\documentclass{article}
\usepackage[numbers]{natbib}

\usepackage{mathtools}
\usepackage[preprint]{neurips_2026}
\usepackage[utf8]{inputenc} 
\usepackage[T1]{fontenc}    
\usepackage{hyperref}      
\usepackage{url}            
\usepackage{booktabs}       
\usepackage{amsfonts}      
\usepackage{nicefrac}       
\usepackage{microtype}    
\usepackage{xcolor}        
\usepackage{amsmath}
\usepackage{cleveref}     
\usepackage{graphicx}
\usepackage{algorithm}
\usepackage{algorithmic}
\usepackage{amsthm}
\usepackage{amssymb}

\usepackage{todonotes}

\usepackage{subcaption}

\newtheorem{theorem}{Theorem}[section]

\newtheorem{lemma}[theorem]{Lemma}

\newtheorem{definition}[theorem]{Definition}
\newtheorem{assumption}[theorem]{Assumption}

\usepackage{natbib}

\DeclareMathOperator*{\argmin}{arg\,min}

\definecolor{citeblue}{RGB}{0,102,204}    
\hypersetup{
    colorlinks=true,
    citecolor=citeblue,
    linkcolor=citeblue,
    urlcolor=citeblue,
}

\title{Differentially Private Hyperparameter Tuning using Local Bayesian Optimization}

\author{%
  Getoar Sopa\thanks{Equal contribution} \\
  Department of Statistics\\
  Columbia University\\
  New York, NY \\
  \And
  Juraj Marusic\textsuperscript{*} \\
  Department of Statistics\\
  Columbia University\\
  New York, NY \\
  \AND
  Marco Avella Medina \\
  Department of Statistics\\
  Columbia University\\
    New York, NY \\
  \And
  John P. Cunningham \\
  Department of Statistics\\
  Columbia University\\
    New York, NY \\
}

\begin{document}

\maketitle
\begin{abstract}
Hyperparameter tuning is a key component of machine learning procedures, but when validation data contain sensitive user information, search mechanisms can leak private information through the selected configuration. 
Existing differentially private hyperparameter tuning methods often rely on near-random search, while prior differentially private Bayesian optimization approaches are typically global and, therefore, scale poorly with the hyperparameter dimensionality. 
We study differentially private hyperparameter tuning using local Bayesian optimization, focusing on settings where the validation objective is available only through noisy black box evaluations and gradients are unavailable or impractical to compute. 
We introduce DP-GIBO, a differentially private local Bayesian optimization framework that privately approximates gradients using a Gaussian Process surrogate. 
Under suitable conditions, we prove that DP-GIBO converges to a locally optimal hyperparameter configuration up to a privacy-dependent error, with dimensional dependence that is polynomial rather than exponential.
Empirically, we show that DP-GIBO provides scalable private hyperparameter tuning across multiple tasks, substantially outperforming non-private random search and global Bayesian optimization baselines in moderate-to-high-dimensional hyperparameter spaces.
\end{abstract}

\section{Introduction}
\label{introduction}

Differentially private (DP) machine learning methods have increasingly seen use by practitioners due to the sensitive nature of the data used in some machine learning applications. Examples include healthcare data \citep{kourou2015machine}, data collected at scale by companies that must comply with privacy regulations \citep{cummings2018role}, and more. Differential Privacy allows practitioners to reason about the amount of `data leakage' to an individual by the inclusion of one data point to the input of an algorithm \citep{dwork2006differential}. To preserve these guarantees end-to-end, hyperparameters must also be tuned under differential privacy, since the choice of hyperparameters can itself reveal information about the underlying data.

Independent of this body of research, there have been many recent advances in optimizing black box functions using Bayesian optimization (BO) \citep{shahriari2015taking}. Often, these black box functions appear in the optimization of hyperparameters of machine learning algorithms \citep{snoek2012practical}, reinforcement learning \citep{brochu2010tutorial, muller2021local}, experimental design \citep{greenhill2020bayesian}, and more \citep{shahriari2015taking}. 
These methods allow for the efficient evaluation of points in the domain space, and therefore often find `good' parameter configurations in fewer rounds of training than classical grid search methods. Unfortunately, classical \emph{global} BO fits a single surrogate over the entire search space and selects each query to optimize a global acquisition function, an approach that breaks down in high dimensions, rendering unrestricted global optimization nearly intractable for as few as $10$ dimensions \citep{garnett2023bayesian, eriksson2021high}. \emph{Local} Bayesian optimization approaches sidestep this by restricting attention to a neighborhood of the current iterate: TuRBO \citep{eriksson2019scalable} maintains many small trust regions, while Gradient Informative Bayesian Optimization (GIBO) \citep{muller2021local, wu2024behavior} uses BO to approximate the local gradient and descends along it.

A unifying question, then, is: \textit{how can we tune hyperparameters in a differentially private way}? Recent work has mostly considered nonadaptive, random grid search methods \citep{liu2019private,papernot2021hyperparameter, mohapatra2022role}. These papers have developed powerful techniques for ensuring a constant privacy cost for training the model for given hyperparameter configurations, even as a random---potentially large---number of configurations are evaluated.  Recently, this line of work was extended to adaptive hyperparameter search strategies by \citet{wang2023dp}. These methods have in common that in order to achieve a small fixed privacy cost in higher dimensions, the hyperparameter sampling mechanism has to be close to uniformly random, making the hyperparameter search intractable in these problems. Recent work has also laid a theoretical foundation for efficient differentially private bilevel optimization for structured smooth bilevel problems with access to first- and second-order derivative information \citep{kornowski2024differentially, lowy2025differentially}, though this information is often unavailable or infeasible to compute.

\begin{figure*}[t]
  \centering
  \includegraphics[width=\textwidth]{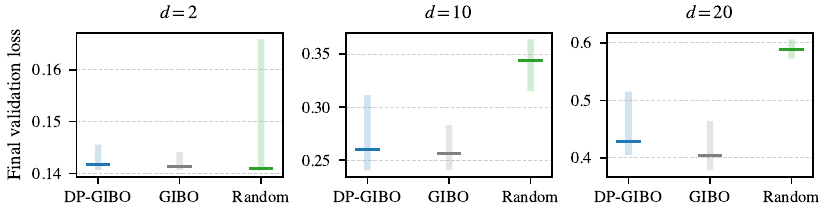}
  \caption{\textit{DP-GIBO scales to higher-dimensional private hyperparameter tuning.} Final validation loss on a group LASSO tuning task with $d \in \{2, 10, 20\}$ group regularization hyperparameters, where each evaluation refits the group LASSO on the training split and queries the validation loss on held-out users, yielding a black box objective with no usable gradient w.r.t.\ hyperparameter. Random search is competitive in low dimensions but degrades sharply as $d$ grows, and offers no privacy protection by default, so any released hyperparameter is a function of the validation data and may reveal information about individual users. In contrast, DP-GIBO closely tracks its non-private counterpart (GIBO) throughout, paying only a modest privacy cost.}
  \label{fig:group-lasso-plots}
  \vspace{-0.4cm}
\end{figure*}

On the other hand, \citet{kusner2015differentially} devised a framework for selecting candidate hyperparameter configurations using Bayesian optimization when first- and second-order derivative information is unavailable, while maintaining privacy in the validation dataset. Making BO differentially private was a significant development, and the authors show that their algorithm provided rapid convergence to a minimizing parameter configuration for certain classes of problems.  However, as a classic global Bayesian Optimization procedure, it too suffers in higher dimensions.

Here we introduce a differentially private analog of GIBO (DP-GIBO), which can be considered a black box approximation of Noisy Gradient Descent when gradients are unknown or prohibitively expensive to compute, and hence unavailable to the practitioner. 
Our algorithm enables efficient, private optimization over high-dimensional continuous spaces, thus filling the gap left by traditional global Bayesian optimization or random search routines.
This makes DP-GIBO applicable to a wide range of private hyperparameter tuning problems where the validation loss is queried as a black box---ranging from a single regularization parameter in vanilla LASSO, to per-group penalties in group LASSO, to the joint tuning of regularization strengths and kernel length-scales in kernel methods, where the hyperparameter dimension can easily exceed one hundred.
We prove that, for sufficiently well-behaved objectives, the suboptimality gap of the algorithm's final estimate converges at a rate that depends polynomially on the problem dimensionality, up to an error scaling with the magnitude of the injected privacy noise.
\Cref{fig:group-lasso-plots} illustrates this scaling behavior on a private group LASSO tuning task, where the curse of dimensionality clearly affects random search but not DP-GIBO.

\paragraph{Contributions.} In summary, in this work we:
(1) Propose DP-GIBO, a drop-in differentially private algorithm for tuning (potentially high-dimensional) continuous hyperparameters in a black box manner.
(2) Prove convergence rates for the suboptimality gap of DP-GIBO under standard regularity conditions, with explicit dependence on dimension and privacy noise.
(3) Empirically demonstrate that DP-GIBO substantially outperforms existing private black box tuning methods in higher-dimensional regimes where global Bayesian optimization and random search collapse.
(4) Introduce a modification to the GIBO framework that enables tighter error control and yields a substantial computational speedup in moderate dimensions.
\section{Preliminaries and Motivation}
\label{preliminaries}
This section is organized as follows. We first provide a primer on Differential Privacy and Gaussian Processes. We then formally define the problem setting, using the private tuning of regularization hyperparameters in a group LASSO as a motivating example. Finally, we position our framework at the intersection of Bayesian optimization, noisy gradient descent, and private hyperparameter tuning.

\subsection{Differential Privacy}
We adopt the Gaussian Differential Privacy (GDP) framework \citep{dong2022gaussian}. Let $\mathcal{M}: \mathcal{X}^n \to \mathcal{Y}$ denote a randomized mechanism taking as input $\mathbf{x} = (x_1, \dots, x_n) \in \mathcal{X}^n$, and call $\mathbf{x}, \mathbf{x}'$ neighboring datasets if they differ in exactly one entry. The following definitions characterize GDP.
\begin{definition}
     Given two probability distributions $P$ and $Q$ and a rejection rule $0\leq \phi \leq 1$, the trade-off function is defined as
    $
    T(P,Q)(\alpha) = \inf \{1 - \mathbb{E}_Q[\phi]:  \mathbb{E}_P[\phi] \leq \alpha\},
    $
    where the infimum is taken over the set of measurable rejection rules.
\end{definition}

\begin{definition}
     A randomized mechanism $\mathcal{M}$ is $\mu$-GDP for $\mu>0$ if, for all neighboring datasets $\mathbf{x}$ and $\mathbf{x}'$, the mechanism satisfies
    $
    T\left(\mathcal{M}(\mathbf{x}), \mathcal{M}(\mathbf{x}')\right)(\alpha) \geq T\left(N(0,1), N(\mu,1)\right) (\alpha), \quad\forall \alpha \in[0,1].
    $
\end{definition}
Equivalently, a mechanism is $\mu$-GDP if distinguishing its output on neighboring datasets is at least as hard as distinguishing $\mathcal{N}(0,1)$ from $\mathcal{N}(\mu,1)$, providing a clean hypothesis testing interpretation of the privacy guarantee \citep{wasserman2010statistical}.
To achieve our privacy result, we rely on the Gaussian Mechanism and composition of GDP.
\begin{definition}
    The global sensitivity $GS(h)$ of a $d$-dimensional function $h$ is defined as 
    \(
    GS(h) = \sup_{\mathbf{x}, \mathbf{x}'}\|h(\mathbf{x}) - h(\mathbf{x}')\|,
    \)
    where the supremum is taken over all  neighboring datasets $\mathbf{x}$ and $\mathbf{x}'$. 
\end{definition}

\begin{theorem}[Gaussian Mechanism \cite{dong2022gaussian}]
    \label{gaussianmechanism}
    Let $h$ be a deterministic function with finite global sensitivity $GS(h)$. The randomized function $\tilde{h}(x) = h(x) + \tfrac{GS(h)}{\mu}Z$, where $Z \sim \mathcal{N}(0,I)$, is $\mu$-GDP. 
\end{theorem}

\begin{theorem}[GDP Composition \cite{dong2022gaussian}]
    \label{composition}
    If the algorithms $\mathcal{A}_t$ are $\mu_t$-GDP for $1\leq t\leq T$, then the $T$-fold composition of $\mathcal{A}_1$, \ldots, $\mathcal{A}_T$ is $\sqrt{\sum_{t=1}^T \mu_t^2}$-GDP.
\end{theorem}

\subsection{Gaussian Processes (GPs)}

A Gaussian Process $\mathcal{GP}(m, k)$ is a stochastic process on $\Theta \subset \mathbb{R}^d$, characterized by a mean function $m: \Theta \to \mathbb{R}$ and a kernel function $k: \Theta \times \Theta \to \mathbb{R}$. If $f \sim \mathcal{GP}(m, k)$, then for any finite collection of points $\mathcal{D} = (\theta_1, \dots, \theta_b) \in \Theta^b$, the vector of evaluations $f(\mathcal{D}) := (f(\theta_1), \dots, f(\theta_b))^\top$ is jointly Gaussian:
\[
    f(\mathcal{D}) \sim \mathcal{N}\!\left(m(\mathcal{D}),\, k(\mathcal{D}, \mathcal{D})\right),
\]
where $m(\mathcal{D}) := (m(\theta_1), \dots, m(\theta_b))^\top \in \mathbb{R}^b$ is the mean vector, and we overload $k(\mathcal{D}, \mathcal{D}) \in \mathbb{R}^{b \times b}$ to denote the kernel matrix with entries $k(\mathcal{D}, \mathcal{D})_{ij} = k(\theta_i, \theta_j)$.

Conditioning preserves the GP structure: given function observations at points $\mathcal{D} \in \Theta^b$, the posterior is $f \mid (\mathcal{D}, f(\mathcal{D})) \sim \mathcal{GP}(m_\mathcal{D}, k_\mathcal{D})$, with
\[
\begin{aligned}
m_\mathcal{D}(\theta) &= m(\theta) + k(\theta, \mathcal{D})\, k(\mathcal{D}, \mathcal{D})^{-1} \bigl(f(\mathcal{D}) - m(\mathcal{D})\bigr), \\
k_\mathcal{D}(\theta, \theta') &= k(\theta, \theta') - k(\theta, \mathcal{D})\, k(\mathcal{D}, \mathcal{D})^{-1}\, k(\mathcal{D}, \theta').
\end{aligned}
\]

Finally, our analysis uses the fact that if $f \sim \mathcal{GP}(m, k)$ with a twice-differentiable kernel, then $f$ and its gradient $\nabla f$ form a joint Gaussian Process. In particular, suppose that at a collection of points $\mathcal{D} \in \Theta^b$ we observe noisy values $\mathbf{y} := f(\mathcal{D}) + \boldsymbol{\varepsilon}$ with $\boldsymbol{\varepsilon} \sim \mathcal{N}(0, \sigma^2 I)$. Then for any $\theta \in \Theta$, the joint distribution of $\mathbf{y}$ and $\nabla f(\theta)$ is multivariate Gaussian:
\[
\begin{pmatrix}
    \mathbf{y} \\
    \nabla f(\theta)
\end{pmatrix}
\sim \mathcal{N}\!\left(
\begin{pmatrix}
    m(\mathcal{D}) \\
    \nabla m(\theta)
\end{pmatrix},
\;\Sigma(\mathcal{D}, \theta)
\right),
\]
with
\[
\Sigma(\mathcal{D}, \theta) =
\begin{bmatrix}
    k(\mathcal{D}, \mathcal{D}) + \sigma^2 I 
      & k(\mathcal{D}, \theta)\,\nabla^\top \\
    \nabla k(\theta, \mathcal{D}) 
      & \nabla k(\theta, \theta)\,\nabla^\top
\end{bmatrix},
\]
where $k(\theta, \mathcal{D})$ denotes the row vector with $j$'th entry $k(\theta, \theta_j)$ for $j \in [b]$, and $\nabla k(\theta, \theta)\nabla^\top$ denotes the matrix with $ij$'th entry $\partial_{\theta_i} \partial_{\theta'_j}k(\theta, \theta')\mid_{\theta' = \theta}$ for $i,j \in [d]$.
\subsection{Problem setting}
Let $\mathcal{X}$ be a data space containing sensitive user data, and let $\Theta \subset \mathbb{R}^d$ be a convex, compact hyperparameter search space. We aim to minimize a validation function $f_{\mathbf{x}}: \Theta \to \mathbb{R}$ of the form
\begin{equation} \label{eqn::optimization-function}
    f_{\mathbf{x}}(\theta) = \tfrac{1}{n}\textstyle\sum_{i=1}^n \mathcal{L}(\theta, x_i),
\end{equation}
where $\mathbf{x} := (x_1, \ldots, x_n) \in \mathcal{X}^n$ is the data of $n$ users and $\mathcal{L}(\cdot, x): \Theta \to \mathbb{R}$ is a \emph{black box} function: given $\theta \in \Theta$, we can only query noisy evaluations $y_i := \mathcal{L}(\theta, x_i) + \varepsilon_i$ with $\varepsilon_i \overset{\text{iid}}{\sim} \mathcal{N}(0, \sigma^2)$, and we \textit{do not} have access to the gradient $\nabla_\theta \mathcal{L}$. The additive Gaussian noise models randomness in the observed output, e.g.\ from differentially private training.

The black box objective functions involving user data are often of the form given in \eqref{eqn::optimization-function}.  For instance, this arises in hyperparameter tuning, where \( \mathcal{L}(\theta, x_i) \) represents the validation loss of the \(i\)-th user in the validation set. Examples of this include the regularization parameters in elastic net regression, the regularizer and kernel length scales in kernel Support Vector Machines (SVMs), and the group regularization terms in group LASSO, as we elaborate on in \Cref{sec:motivating-example}.

While hypergradients can sometimes be computed in smooth bilevel problems, many practical tuning pipelines expose only validation losses at chosen hyperparameter values. In examples such as kernel SVMs and group LASSO, nonsmooth losses, constraints, and active-set changes can make exact differentiation through training unreliable or impractical. We therefore focus on the black box regime, in which the algorithm queries possibly noisy user-level validation losses rather than gradients.

While hyperparameter tuning is our focus, the form \eqref{eqn::optimization-function} also arises in (federated) reinforcement learning, where the objective is expected cumulative reward under a given policy \citep{muller2021local, qi2021federated}, and in personalized federated learning, where it measures average post-adaptation performance of a shared model initialization across users \citep{fallah2020-personalized-federated}.

\subsection{Motivating example: Private hyperparameter tuning in group LASSO}
\label{sec:motivating-example}

To make the problem setting concrete, consider a group LASSO regression problem \citep{yuan2006model} in which a practitioner 
wishes to tune the regularization strength applied to each feature group while assuming the training procedure used to compute $\widehat w(\theta)$ is public or separately privatized, while preserving the privacy of the validation users used for hyperparameter selection. Given a design matrix $X^{\text{train}} \in \mathbb{R}^{n_\text{train} \times p}$ whose columns are partitioned into $d$ disjoint groups $G_1, \dots, G_d$, and a response vector $y^{\text{train}} \in \mathbb{R}^{n_\text{train}}$, the group LASSO estimator solves 
\begin{equation*}
    \hat{w}(\theta) \in \argmin_{w \in \mathbb{R}^p} \tfrac{1}{2n_\text{train}}\|y^{\text{train}} - X^{\text{train}}w\|_2^2 + \textstyle\sum_{k=1}^{d} \theta_k \|w_{G_k}\|,
\end{equation*}
where $\theta = (\theta_1, \dots, \theta_d) \in \Theta \subset \mathbb{R}_{>0}^d$ assigns one regularization parameter per group \citep{simon2013sparse, obozinski2011group}.

In a typical hyperparameter tuning workflow, one selects $\theta$ by minimizing the validation loss on a held-out validation set, disjoint from the training data used to fit $\hat{w}(\theta)$. This corresponds exactly to the objective in \eqref{eqn::optimization-function} with
\[
    \mathcal{L}(\theta, x_i)
    \equiv
    \mathcal{L}(\theta, (x_i^{\text{val}}, y_i^{\text{val}}))
    =
    \tfrac{1}{2}\bigl(y_i^{\text{val}} - \langle x_i^{\text{val}},\, \hat{w}(\theta) \rangle\bigr)^2,
\]
where $i \in [n]$ indexes the $n$ validation users and $(x_i^{\text{val}}, y_i^{\text{val}})$ is the $i$-th user's feature-target pair.

The map  $\theta\mapsto \widehat w(\theta)$ is defined implicitly through a nonsmooth inner optimization problem, and changes in $\theta$ can alter the active set of selected groups. Differentiating the resulting validation loss through a generic solver is therefore nontrivial, motivating a black box treatment in which $\mathcal L(\cdot,x_i)$ is queried only through possibly noisy evaluations.
\subsection{Related work}
\paragraph{Private Hyperparameter Tuning \& BO.} There have been two primary lines of work on hyperparameter tuning under Differential Privacy constraints. The first of these lines \citep{liu2019private, mohapatra2022role, papernot2021hyperparameter}, uses a random number of iteration and a uniformly random search over the hyperparameters to optimize the hyperparameter configuration at a fixed privacy cost. Recently, \citet{wang2023dp} extended this line of work to the adaptive setting, shedding the need for a uniformly random search over the parameter space. 
As the dimension of the domain increases, however, the class of sampling distributions that can achieve a small privacy cost in this algorithm approaches the uniform distribution, preventing scalability.

The second line of work, introduced by \citet{kusner2015differentially}, explicitly adapts a Bayesian optimization procedure to the differentially private setting. As the adapted procedure, Upper Confidence Bound (UCB) BO, is a global method, it also does not scale well in the number of optimizable parameters.
This work has been followed by a diverse literature on differentially private Bayesian optimization, introducing variations to the privacy structure \citep{kharkovskii2020private, zhou2021local, dai2021differentially}. Furthermore, there is a broad literature on privatizing common Multi-Armed Bandit algorithms, which addresses a similar problem as Bayesian optimization but treats candidate hyperparameters as independent arms \citep{tossou2016algorithms, shariff2018differentially,  ou2024thompson}.

\paragraph{Bayesian Optimization.} There have been many recent advances in adapting Bayesian optimization techniques to high-dimensional spaces. \citet{eriksson2021high} introduce SAASBO, which identifies sparse axis-aligned subspaces to optimize over, while \citet{hvarfner2024vanillaicml} addresses the curse of dimensionality by placing an increasing prior on the surrogate GP's lengthscale. These approaches both assume that the high-dimensional function is not as complex as the dimensionality would suggest. Using a zeroth-order Bayesian optimization framework to perform approximate gradient descent along black box functions is another recent idea that was introduced by \cite{muller2021local}. 
Certain properties of this approach were proved in \citet{wu2024behavior}, under the assumption that the black box function is a draw from the GP. 

\paragraph{Noisy Gradient Descent.}On the other hand, gradient descent algorithms are among the algorithms that have received the most attention in the Differential Privacy literature, when gradients are known or practically computable (in \textit{contrast} to the setting of this work). Noisy (Stochastic) Gradient Descent algorithms are among the most widely used differentially private empirical risk minimizers \citep{privateSGD, privateGD3,abadietal2016,iyengaretal2019,privateGD2, songetal2021}. Recent works such as \cite{avella2023differentially,ganeshetal2023,yu:zhao:zhou2024} provide statistical rates for private noisy first- and second-order methods. 

\section{Differentially Private Gradient Informative Bayesian Optimization}
\label{DPLBO}

\subsection{Algorithm}
Our procedure, detailed in Algorithm \ref{algorithm}, uses Bayesian optimization to approximate the underlying function gradient by the gradient of the surrogate model, performs an appropriate clipping step, adds noise, and performs gradient descent along this direction. The Bayesian optimization step consists of sampling and evaluating points that minimize a so-called \textit{acquisition function} over the domain in order to extract as much information relevant to the optimization procedure as possible. The acquisition function we use in this Bayesian optimization framework is the following:
\begin{equation}
    \alpha(\mathbf{z}; \mathcal{D}, \theta) = \text{Tr}\Bigl( \nabla k(\theta, \theta)\nabla^\top - \nabla k(\theta, \mathcal{D} \cup \mathbf{z}) \bigl( k(\mathcal{D} \cup \mathbf{z}, \mathcal{D} \cup \mathbf{z}) + \sigma^2 I \bigr)^{-1} k(\mathcal{D} \cup \mathbf{z}, \theta)\,\nabla^\top \Bigr),
\end{equation}
where $\mathbf{z}$ is a set containing points in $\Theta$.
Intuitively, this acquisition function aims to minimize the current step's uncertainty in the posterior gradient covariance, given (potentially noisy) new observations. 
Thus, choosing points $\mathbf{z}$ that minimize this acquisition function implies choosing points that minimize the uncertainty in the gradient approximation at the point $\theta$. In this sense, it can also be understood as a greedy gradient bias minimizer. 
\begin{algorithm}[t]
\caption{Differentially Private Local Bayesian Optimization}
\begin{algorithmic}[1]
\STATE \textbf{Input:} Loss functions $\mathcal{L}(\cdot, x_i)$, clipping constant $B$, error tolerance $\varepsilon$, objective function noise variance $\sigma^2$, step size $\eta$, and privacy parameter $\mu$
\STATE Initialize at a non-data dependent $\theta^{(0)} \in \Theta$, initialize evaluation set $\mathcal{D}_{-1} = \emptyset$
\FOR{$t = 0, \dots, T-1$}
    \STATE $b = \min \left \{ b' \in \mathbb{N}^+: \min_{\mathbf{z}: |\mathbf{z}| = b'} \alpha(\mathbf{z}; \mathcal{D}_{t-1}, \theta^{(t)}) \leq \varepsilon \right \}$ \hfill \textcolor{gray}{\textit{// gradient approximation start}}
    \STATE Find the set of $b$ points minimizing the acquisition function,
    \[\mathbf{z} = \arg\min_{\mathbf{z}': |\mathbf{z}'| = b} \alpha(\mathbf{z}'; \mathcal{D}_{t-1}, \theta^{(t)}),\]
    and obtain noisy evaluations $\mathbf{y}_i^{(t)} \in \mathbb{R}^b$ with $(\mathbf{y}_i)_j^{(t)} = \mathcal{L}(\mathbf{z}_j, x_i) + \varepsilon_{ij}$, $\varepsilon_{ij} \sim \mathcal{N}(0, \sigma^2)$ 
    \STATE Update $\mathcal{D}_{t} = \mathcal{D}_{t-1} \cup \mathbf{z}$
    \STATE Compute the per-user surrogate posterior mean gradient at $\theta^{(t)}$: \hfill \textcolor{gray}{\textit{// approximation end}}
    \[
    g_t^{(i)} := \nabla k(\theta^{(t)},\mathcal{D}_t) \left(k(\mathcal{D}_t, \mathcal{D}_t) + \sigma^2 I\right)^{-1}[(\mathbf{y}_i^{(0)})^\top, \ldots, (\mathbf{y}_i^{(t)})^\top]^\top
    \] 
    \STATE Clip and aggregate to form an approximate gradient: \hfill \textcolor{gray}{\textit{// per-user clipping}}
    \[
        g_t := \tfrac{1}{n} \textstyle\sum_{i=1}^n g_t^{(i)} \cdot \min\bigl\{1,\, \tfrac{B}{\|g_t^{(i)}\|}\bigr\}
    \]
    \STATE Draw $w_t \overset{\text{iid}}{\sim} \mathcal{N}(0, I)$ and perform the noisy descent step: \hfill \textcolor{gray}{\textit{// privacy noise injection}}
    \[
        \theta^{(t+1)} = \theta^{(t)} - \eta \!\left(g_t + \tfrac{2B\sqrt{T}}{n\mu}w_t\right)
    \]
\ENDFOR
\STATE \textbf{Return} $\theta^{(T)}$
\end{algorithmic}
\label{algorithm}
\end{algorithm}

In each iteration of the algorithm, our algorithm chooses the minimum number of required points for evaluation necessary to push the gradient bias below some specified threshold $\varepsilon$. It then uses the noisy function evaluations at this point and all previously evaluated points to approximate the gradient at that iterate for all validation users, clips and aggregates these, and performs a step of noisy gradient descent along this approximate gradient descent direction. 

We draw particular attention to Line 4, which did not appear in earlier GIBO implementations \citep{muller2021local, wu2024behavior} and carries multiple advantages.
Computationally, Line 4 is an optimization over the kernel function alone, requiring no new model training or validation.
Since posterior variance reduction in GP design is approximately submodular, we solve Lines 4--5 greedily: we incrementally grow the batch size $b$ (binary search is also possible), warm-start from the previous solution, and refine via L-BFGS.
Statistically, Line 4 controls the gradient bias $\|g_t - \nabla f_\mathbf{x}(\theta^{(t)})\|$, which is what enables (i) rigorous convergence guarantees and (ii) analysis under local strong convexity.
It also prevents superfluous function evaluations, speeding up optimization; when each evaluation requires a DP training routine, this further reduces the training privacy overhead. 
\subsection{Theoretical Guarantees}\label{section:noisy}
Without making any assumptions about our data or parameters, we can state the following privacy guarantee for Algorithm \ref{algorithm}.

\begin{theorem}
\label{thm:privacyguarantee}
    The vector $(\theta^{(1)}, \dots , \theta^{(T)})$ released by Algorithm \ref{algorithm} is $\mu$-GDP. In particular, the algorithm's final estimate $\theta^{(T)}$ is $\mu$-GDP.
\end{theorem}

We now turn to analyzing the utility of our algorithm when $\sigma^2>0$, as is the case, for example, when training is done differentially privately by introducing a surrogate Gaussian Process. In this setting, $g_t^{(i)}$ is the posterior mean of validation user $i$'s validation loss gradient, and the trace of the posterior covariance controls how large the deviation of the true gradient can be from this estimate.  In Appendix \ref{section:convergencenoiseless} we analyze the noiseless case using RKHS interpolation. As we have $n$ validation functions, we assume that each of them is distributed as a Gaussian Process.
\begin{assumption}\label{assumption:GP}
    The functions $\mathcal{L}(\cdot, x_i)$ satisfy $\mathcal{L}(\cdot, x_i) \overset{iid}\sim \mathcal{GP}(\mathbf{0}, k)$,
    where $k$ is the kernel used in Algorithm \ref{algorithm}, and we assume that $k$ is non-constant, positive definite, stationary, isotropic, and $k \in C^{4+\alpha}$ for some $\alpha>0$. 
\end{assumption}
The class of kernels described in Assumption \ref{assumption:GP} is broad and includes the RBF kernel and sufficiently smooth Matern($\nu$) kernels. The kernel assumptions made in $\mathcal{L}(\cdot,x_i)$ automatically imply several appealing properties that we use to prove the convergence guarantees of our Algorithm. In particular, they imply with high probability the $L$-smoothness of the global validation function $f_\mathbf{x}$ in \eqref{eqn::optimization-function}, and the boundedness of the gradients of the local validation functions $\mathcal{L}(\cdot, x_i)$:
\begin{lemma}[Smoothness and bounded gradients]\label{lemma:placeholder}
 Under Assumption \ref{assumption:GP}, given a $0<\delta<1$, there exist $L(\delta), B(\delta) <\infty$ such that with probability at least $1-\delta$,\[ \|\nabla f_\mathbf{x}(\theta) - \nabla f_\mathbf{x}(\theta')\| \leq L(\delta)\|\theta - \theta'\| \quad \text{ and  } \quad\|\nabla \mathcal{L}(\theta, x_i)\| \leq B(\delta)\] for all $\theta, \theta' \in \Theta$ and $i \in [n]$.
\end{lemma}
We \textit{assume} the existence of a local minimum in the interior of our hyperparameter search space.
\begin{assumption}\label{assumption: localminimum}
    The realized sample path $f_\mathbf{x}$ has a local minimum $\theta^* \in \text{int}(\Theta)$, i.e. $\nabla f_\mathbf{x}(\theta^*) = 0$.
\end{assumption}
Notice that we do not need to assume that the local minimum is nondegenerate, i.e. $\nabla^2 f_\mathbf{x}(\theta^*) \succ \mathbf{0}$, because under Assumption \ref{assumption:GP}, Bulinskaya’s lemma implies that $f_\mathbf{x}$ is almost surely Morse on $\Theta$ \cite{adler2007random}. Consequently, the local minimum is almost surely nondegenerate. Assumptions \ref{assumption:GP} and \ref{assumption: localminimum} then imply local strong convexity (LSC) in a ball around $\theta^*$:
\begin{lemma}[LSC]\label{lemma:placeh}
    Under Assumptions \ref{assumption:GP} and \ref{assumption: localminimum}, given a $0< \delta <1$, there exist $\tau(\delta), r(\delta)>0$ such that $B_{r(\delta)}(\theta^*) \subset \Theta$ and, with probability at least $1-\delta$,
    \[
    \langle \nabla f_\mathbf{x}(\theta) - \nabla f_\mathbf{x}(\theta'), \theta - \theta'\rangle \geq \tau(\delta) \|\theta - \theta'\|^2
    \]
    for all $\theta, \theta' \in B_{r(\delta)}(\theta^*)$.
\end{lemma}
Lemmas \ref{lemma:placeholder} and \ref{lemma:placeh} suggest that if a local minimum exists in the interior of $\Theta$, then our GP assumption supports $L$-smoothness, bounded gradients, and local strong convexity with high probability, motivating an analysis under these assumptions. 
We are thus able to state our convergence guarantee for the noisy setting. To this end, we  denote $F_t := f_\mathbf{x}(\theta^{(t)}) - f_\mathbf{x}(\theta^{*})$, the suboptimality gap that Algorithm \ref{algorithm} achieves after $t$ iterations. 

\begin{theorem}
\label{theorem:functionEvalsNoisy}
    Let Assumptions \ref{assumption:GP} and \ref{assumption: localminimum} hold.  Furthermore, let $n = \Omega(\frac{\sqrt{T(d+ \log\frac{T}{\delta})}}{\mu})$, $\eta \leq \frac{1}{L(\delta)}$, $B \geq B(\delta)$, and set the gradient bias threshold $\varepsilon = \varepsilon_\omega = d\sigma \omega^{-\frac{1}{2}}$ for $\omega \in \mathbb{N}$ such that $\varepsilon_\omega \leq \frac{\tau(\delta)^2r(\delta)^2}{20\log\frac{2nT}{\delta}}$.
    Then if $\|\theta^{(0)} - \theta^*\| \leq r(\delta)$, taking $T \asymp \log n$, with probability at least $1-2\delta$,
    \[
        F_T = O\Bigg(\frac{\sqrt{\log n(d + \log \frac{\log n}{\delta})}}{n\mu} + \frac{\sigma d \log {\frac{n}{\delta}}}{\sqrt \omega}\Bigg)
    \]
     using $O(\omega d \log n)$ total function evaluations.
\end{theorem}

The first term in the high-probability suboptimality gap bound of Theorem \ref{theorem:functionEvalsNoisy} is induced by the privacy noise, whereas the second term is induced by the bias of approximating the validation loss function's gradient using a surrogate GP. We note that in many practical applications, $\sigma$ can be expected to decrease with $n$, often at a rate of $\sigma = \tilde O(\frac{\text{poly}(d)}{n})$, for example, when the function evaluation noise is induced by DP training on a training dataset of size $\Omega(n)$, and the validation loss is Lipschitz in the underlying model parameters.  Furthermore, the bound has two implications. First, the computational burden is effectively bounded by the privacy floor: given a privacy cost induced by $n$, $\mu$, and $d$, increasing $\omega$ beyond the point where the final approximation bias is of a smaller order than the privacy error floor comes with no gain to the achieved suboptimality gap. Hence, there is a computational saturation point beyond which additional function evaluations do not improve the order of the final guarantee. On the other hand, if computation is limited to at most $\omega_{\max} d \log n$ total function evaluations, then if ${\sqrt{\omega_{\max} \log n (d + \log \frac{\log n}{\delta})}}{} = o(\mu n\sigma d\log \frac{n}{\delta})$, the privacy error becomes negligible relative to the bias, so privacy carries no asymptotic cost.

\paragraph{Comparison to global DP BO.} \citet{kusner2015differentially} analyze a global differentially private Bayesian optimization procedure over a finite discretization of the search space, and bound the high-probability \emph{global} suboptimality gap after $\omega$ function evaluations. Combining their utility bound with the standard RBF-kernel GP-UCB regret rate yields, in $\mu$-GDP terms, a rate of order
\[
    \tilde{O}\!\left( \frac{e^{(d/2)\log\log\omega}\sqrt{d}}{\sqrt{\omega}} + \frac{d^{3/2}}{\mu} \right),
\]
where the $e^{(d/2)\log\log\omega}$ factor in the first term reflects the curse of dimensionality entering through the GP-UCB information gain. By contrast, DP-GIBO scales polynomially in $d$ up to log factors (\Cref{theorem:functionEvalsNoisy}), at the cost of converging only to a local minimum rather than a global one.

We note that this local-vs-global trade-off is less severe in practice than it might appear: in high-dimensional non-convex optimization, local minima tend to concentrate near the global minimum in objective value \citep{ dauphin2014identifying}. Local methods therefore often recover near-global solutions in regimes where global methods are computationally intractable.

\vspace{-0.5em}

\section{Examples}
\label{examples}
We now present three numerical experiments to illustrate the behavior of our proposed algorithm and to compare it to other private algorithms and to non-private Bayesian optimization procedures. More detailed discussions about these examples, along with additional studies, can be found in Appendix \ref{appendix:examples}. All experiments were run locally on an Apple Macbook with an M3 Pro Chip with 11 Core CPU and 14 Core GPU and 36GB of Unified Memory.

\subsection{Hyperparameter tuning in group LASSO} \label{example::group-LASSO-main}
We now revisit the group LASSO motivating example of \Cref{sec:motivating-example} as our first empirical setting. \Cref{fig:group-lasso-plots} reports performance across hyperparameter dimensions $d \in \{2, 10, 20\}$, where each method is given the same budget of function evaluations. Differentially private hyperparameter tuning is already challenging for vanilla LASSO, and group LASSO is harder: the per-group regularization parameters expand the hyperparameter space and rule out one-dimensional grid-search baselines. 

At $d=2$ all three methods are competitive: the search space is small enough that uniform sampling finds configurations comparable to those identified by GIBO and DP-GIBO. We note that random search itself is not private by default, as the released hyperparameter is still a function of the validation data. As $d$ grows, random search degrades sharply, while GIBO and DP-GIBO continue to make meaningful progress at $d=10$ and $d=20$ by exploiting local gradient information rather than relying on global coverage. Notably, DP-GIBO tracks its non-private counterpart closely across all dimensions, indicating that the privacy cost of our method does not compound with $d$, as predicted by \Cref{theorem:functionEvalsNoisy}, even though the assumptions of that theorem are not satisfied in this setting.

A natural point of comparison would be DP-HyPO \citep{wang2023dp}, but a public implementation is not currently available and the experiments reported there are limited to low-dimensional spaces (notably 2-D in their main DP-SGD benchmark). Since DP-HyPO can only guarantee strong privacy guarantees in high dimensions for nearly-uniformly random search procedures, we would expect it to perform comparably to random search in the regimes considered here, consistent with the curse of dimensionality observed in \Cref{fig:experiments-plots}(a) for global private methods. 

\subsection{Hyperparameter tuning in Gaussian Process Regression}
\label{example:GP}

In this example, we examine the performance of Algorithm \ref{algorithm} when tuning the kernel length-scales of a Gaussian Process regression with one length-scale per input dimension. We treat the validation loss as a black box function of $\theta \in \mathbb{R}^d$, where $\theta$ collects the per-dimension log-length-scales, and we sweep the dimensionality across $d \in \{2, 5, 10\}$. \Cref{fig:experiments-plots}(a) reports the resulting optimization trajectories: at low dimensions, all three methods are competitive, but as $d$ grows, DP-GIBO continues to make rapid progress while Random Search and Global BO plateau early at higher validation losses.

The Global BO baseline is UCB-BO \citep{srinivas2010gaussian} in its standard, non-private form; making it differentially private would require sampling the next query from a distribution over all observed points with probabilities calibrated to their acquisition values \citep{kusner2015differentially}, so the curves shown here serve as an optimistic lower bound on what a private UCB-BO could achieve.

\subsection{Hyperparameter tuning of a Kernel SVM}
\label{sec:svmtuning}
To assess our method on real data, we tune the hyperparameters of a Kernel SVM on the CT slice regression dataset \citep{relative_location_of_ct_slices_on_axial_axis_206}. This yields a $d = 103$ dimensional optimization problem over 3 regularization parameters and 100 kernel length-scales. As shown in \Cref{fig:experiments-plots}(b), DP-GIBO closely tracks the non-private GIBO baseline even at $\mu = 1$ and clearly outperforms random search throughout. The private suboptimality gap remains uniformly larger than the non-private one, in line with the bounds developed in our theory and reflecting the unavoidable cost of differential privacy.

We omit a Global BO baseline in this experiment: at this dataset size exact GP inference is already costly, and it is unclear how to privatize the standard scalable alternatives such as variational and sparse GP approximations. Moreover, as already seen in \Cref{fig:experiments-plots}(a), Global BO performs no better than random search at high dimensions, so we would not expect it to be competitive in this $d = 103$ setting regardless. The objective here likely has lower-dimensional structure, since only a few of the 100 kernel length-scales are expected to be truly active, and recent high-dimensional BO methods such as SAASBO \citep{eriksson2021high} and the lengthscale-prior approach of \citet{hvarfner2024vanillaicml} are designed to exploit precisely this kind of structure; however, it is currently unknown how to privatize these procedures while preserving their statistical benefits. DP-GIBO sidesteps both issues, since its surrogate is fit only over a small batch of locally sampled points at each iteration and never needs to scale with the dataset size.

\begin{figure*}[h]
  \centering
  \begin{subfigure}[t]{0.66\textwidth}
    \centering
    \includegraphics[width=\textwidth]{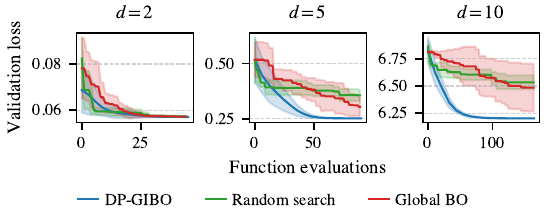}
    \caption{Tuning a GP regression across dimensions $d \in \{2, 5, 10\}$.}
    \label{fig:intro-dims}
  \end{subfigure}
  \hfill
  \begin{subfigure}[t]{0.32\textwidth}
    \centering
    \includegraphics[width=\textwidth]{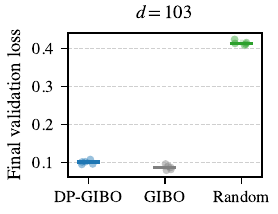}
    \caption{Tuning a Kernel SVM.}
    \label{fig:intro-second}
  \end{subfigure}
  \caption{\textit{Local private BO does not suffer from the curse of dimensionality.} (a) DP-GIBO converges rapidly across all dimensions, while Global BO (UCB-BO) and Random Search plateau at substantially higher losses as $d$ grows. (b) Tuning a Kernel SVM in $d = 103$ dimensions: DP-GIBO closely matches its non-private counterpart GIBO and clearly outperforms Random Search. 
  }
  \label{fig:experiments-plots}
\end{figure*}

\vspace{-0.5em}
\section{Discussion}
Differentially private hyperparameter tuning is a practical necessity whenever models are trained on sensitive user data, yet the toolbox available to practitioners remains remarkably thin. 
Random search with privacy accounting on the number of trials is essentially the only widely deployed option, and existing private Bayesian optimization approaches are global methods that scale poorly with the number of hyperparameters.
In this work we showed that local Bayesian optimization offers a broadly applicable alternative. 
We introduced DP-GIBO, a private variant of GIBO that scales gracefully with dimension. 
Under suitable conditions, we proved that it converges to a neighborhood of the optimum at a rate with polynomial dependence on dimension. 
Empirically, it outperforms random search and global private BO across a range of synthetic and real-data hyperparameter tuning tasks, even in regimes where our theoretical assumptions are not strictly satisfied. 
We see DP-GIBO as a practical tool that can be dropped into existing tuning pipelines whenever validation losses depend on private user data.
\vspace{-0.3cm} 
\paragraph{Limitations.} 
We require the optimization space to be continuous. 
It would be interesting to see how this approach could be combined with mapping the discrete space to a continuous latent space, such as in \citet{NEURIPS2022_ded98d28}. Furthermore, we emphasize that, as with all local approaches, our algorithm targets a local optimum, contrary to the objective of global approaches. Finally, as in DP-SGD and DP-GD, our algorithm requires a gradient clipping threshold $B$ as input. This can be done heuristically, or one can adapt  low-sensitivity private algorithms for setting $B$ in a data-specific manner such as \citet{andrew2021differentially, bu2023automatic}, which comes at a small additional privacy cost. We conduct a sensitivity study of $B$ in \Cref{sensitivityB}.
\vspace{-0.3cm} 
\paragraph{Future work.}
Finally, we believe there are several potential avenues for further work in the direction of this paper:
$(i)$ Investigating whether the random number of iteration approach used in \citet{wang2023dp} could be adapted to this framework.
$(ii)$ Incorporating a small number of discrete parameters into the optimization problem.
$(iii)$ The development of private analogs of alternative higher-dimensional Bayesian optimization methods such as TuRBO \citep{eriksson2019scalable} and SAASBO \citep{eriksson2021high}.

\bibliography{main}
\bibliographystyle{unsrtnat}

\thispagestyle{empty}

\clearpage
\clearpage

\appendix

\section*{Appendix}
Throughout the appendix, we use the notation \[k_{\mathcal{D}, \sigma^2}(\theta, \theta') := k(\theta, \theta') - k(\theta, \mathcal{D})(k(\mathcal{D}, \mathcal{D}) + \sigma^2I)^{-1} k(\mathcal{D}, \theta').\]

\section{Supporting Lemmas}

\begin{lemma}
\label{lem:projection}
        If $\max_{i \in [n]} \sup_{\theta \in \Theta}\|\nabla \mathcal{L}(\theta, x_i)\| \leq B$, we have,  $\forall i \in [n],t \in [T]$,
        \[
        \left \| \Pi_{B}\left (g_t^{(i)}(\theta^{(t)})\right) - \nabla\mathcal{L}(\theta^{(t)}, x_i) \right \|  \leq \left \| g_t^{(i)}(\theta^{(t)}) - \nabla\mathcal{L}(\theta^{(t)}, x_i) \right \| .
        \]
\end{lemma}
\begin{proof}
    First, note that $\Pi_B: \mathbb{R}^d \mapsto \mathbb{R}^d$ is a projection 
    onto a closed convex set, the $L_2$ ball of radius B around the origin. Since $B_B(0)$ is closed and convex, the Euclidean projection $\Pi_B$ is nonexpansive \cite{bauschke2020correction}. By assumption, we have, for all $i \in [n]$, $\theta \in \Theta$,
    \[
    \Pi_B(\nabla \mathcal{L}(\theta, x_i)) = \nabla \mathcal{L}(\theta, x_i),
    \]
    and hence
    \begin{align*}
         \| \Pi_{B}\left (g_t^{(i)}(\theta^{(t)})\right) - \nabla\mathcal{L}(\theta^{(t)}, x_i)  \|  &=   \| \Pi_{B}\left (g_t^{(i)}(\theta^{(t)})\right) - \Pi_B \left(\nabla\mathcal{L}(\theta^{(t)}, x_i) \right) \| \\&\leq \|  g_t^{(i)}(\theta^{(t)})- \nabla\mathcal{L}(\theta^{(t)}, x_i)  \|
    \end{align*}
\end{proof}

\begin{lemma}
\label{lemma:gpboundappendix}
    Admit Assumption \ref{assumption:GP}. Furthermore, let $\theta^{(t)} \in \Theta$ and $\mathcal{D}_t$ be increasing sets containing points in $\Theta$, and $\mathbf{y}_i^{(t)}$ collect the noisy evaluations of $\mathcal{L}(\cdot, x_i)$ at the points in $\mathcal{D}_t$. Then, we have, with probability $1-\xi$, for all $i \in [n]$ and $t \in [T]$,
    \[
        \|\nabla \mathcal L(\theta^{(t)}, x_i) - g_t^{(i)}\|^2 \leq  5 \text{Tr}(\nabla k_{\mathcal{D}_t, \sigma^2}(\theta^{(t)}, \theta^{(t)})\nabla^\top )\log \frac{nT}{\xi}
    \]
\end{lemma}
\begin{proof}
Fix $i\in[n]$ and $t\in[T]$.  By the Gaussian conditioning formula, conditionally on $\mathcal{F}_t = \sigma(\mathcal{D}_t, \theta^{(t)}, \mathbf{y}_i^{(t)})$, the sigma-algebra generated by the realized design and observations, 
\[
\nabla\mathcal L(\theta^{(t)},x_i)-g_t^{(i)} \mid \mathcal{F}_t
\sim
N\left(0,\nabla k_{\mathcal{D}_t, \sigma^2}(\theta^{(t)},\theta^{(t)})\nabla^\top\right),
\]
where we use that the $\mathcal{L}(\cdot, x_i)$ are independent GPs so that conditioning on $\theta^{(t)}$, which depends on $\{\mathbf y_j^{(s)}\} _{j \in [n], s<t}$ provides no information about $\mathcal{L}(\theta^{(t)}, x_i)$ beyond $(\mathcal{D}_t, \mathbf y_i^{(t)})$.
Applying Proposition 1.1 of \citet{hsu2012tail} with $q=\log(nT/\xi)$ yields
\begin{align*}
\mathbb{P}\Big(
\|\nabla\mathcal L(\theta^{(t)},x_i)-g_t^{(i)}\|^2
&>
\operatorname{Tr}(\nabla k_{\mathcal{D}_t, \sigma^2}(\theta^{(t)},\theta^{(t)})\nabla^\top)
+
2\sqrt{\operatorname{Tr}\big((\nabla k_{\mathcal{D}_t, \sigma^2}(\theta^{(t)},\theta^{(t)})\nabla^\top)^2\big)\log\frac{nT}{\xi}}
\\&+
2\|\nabla k_{\mathcal{D}_t, \sigma^2}(\theta^{(t)},\theta^{(t)})\nabla^\top\|\log\frac{nT}{\xi}
\ |\ \mathcal F_t
\Big)
\leq
\frac{\xi}{nT}.
\end{align*}
Taking expectations gives the same unconditional bound. A union bound over $i\in[n]$ and $t\in[T]$ and bounding $\text{Tr}((\nabla k_{\mathcal{D}_t, \sigma^2}(\theta^{(t)},\theta^{(t)})\nabla^\top)^2) \leq \text{Tr}(\nabla k_{\mathcal{D}_t, \sigma^2}(\theta^{(t)},\theta^{(t)})\nabla^\top)^2$ and $\|\nabla k_{\mathcal{D}_t, \sigma^2}(\theta^{(t)},\theta^{(t)})\nabla^\top\| \leq \text{Tr}(\nabla k_{\mathcal{D}_t, \sigma^2}(\theta^{(t)},\theta^{(t)})\nabla^\top)$ yields the claim.
\end{proof}
Next, we will give an upper bound on the bias of our estimated gradient. 
\begin{lemma}
\label{lem:bias_gradient}
    Under Assumption \ref{assumption:GP}, we have, for all $ t \in [T]$, with probability at least $1-\xi$,
    \[
    \|g_t - \nabla f_\mathbf{x}(\theta^{(t)})\| \leq \sqrt{5 \log\frac{2nT}{\xi}\text{Tr}(\nabla k_{\mathcal{D}_t, \sigma^2} (\theta^{(t)}, \theta^{(t)})\nabla^\top)}
    \]
\end{lemma}
\begin{proof}
    We start by using the triangle inequality to go to the individual gradients, then use Lemma \ref{lem:projection} to remove the projection operator, and lastly apply Lemma \ref{lemma:gpboundappendix}:
    \begin{align*}
        \left \| g_t - \nabla f_\mathbf{x}(\theta^{(t)}) \right \| &= \left \| \frac{1}{n} \sum_{i=1}^n \Pi_B \left ( g_t^{(i)} \right ) - \frac{1}{n} \sum_{i=1}^n \nabla \mathcal{L}(\theta^{(t)}, x_i)  \right \| \\
        & \leq \frac{1}{n} \sum_{i=1}^n \left \| \Pi_B \left ( g_t^{(i)} \right ) - \nabla \mathcal{L}(\theta^{(t)}, x_i)  \right \| \\
        & \leq \frac{1}{n} \sum_{i=1}^n \left \| g_t^{(i)} - \nabla \mathcal{L}(\theta^{(t)}, x_i)  \right \| \\
         &\leq \frac{1}{n} \sum_{i=1}^n \sqrt{5 \text{Tr}(\nabla k_{\mathcal{D}_t, \sigma^2}(\theta^{(t)}, \theta^{(t)})\nabla^\top )\log \frac{nT}{\xi}} \\
    &= \sqrt{5 \text{Tr}(\nabla k_{\mathcal{D}_t, \sigma^2}(\theta^{(t)}, \theta^{(t)})\nabla^\top )\log \frac{nT}{\xi}}
    \end{align*}
\end{proof}

\begin{lemma}
    \label{subGaussianProof}
    Let $w_1, .., w_T \overset{iid}{\sim} N(0, I)$. Then, we have, with probability at least $1-\delta$,
    \[
    \max_{t\leq T} \|w_t\| \leq 4\sqrt{d} + 2\sqrt{2\log\frac{T}{\delta}}
    \]
\end{lemma}

\begin{proof}
    The proof follows from \citet{rigollet2023high}'s Theorem 1.19 and an application of the union bound.
\end{proof}

\begin{lemma}
\label{lemma:forwarddiff}[Lemma 12 of \citet{wu2024behavior}]
    Admit Assumption \ref{assumption:GP}. Let $\phi$ be the function such that $k(\mathbf{u}, \mathbf{v}) = \phi(\mathbf{u} - \mathbf{v})$. Then, for any $h>0$, and for $m>1$, there exist $2md$ points making up $\mathcal D$ such that
    \[
    \text{Tr}(\nabla k_{\mathcal{D}, \sigma^2}(\mathbf{0}, \mathbf{0})\nabla^\top) \leq -\sum_{j=1}^d\left(\partial_j^2\phi(\mathbf{0}) + \frac{2(\partial _j \phi(-h\mathbf{e}_j))^2}{(1-\phi(2h\mathbf e_j)) + \frac{\sigma^2}{m}}\right),
    \]
where $\mathbf{e}_j$ are the canonical vectors with $1$ in the $j$-th component and $0$ in all others.
\end{lemma}

\begin{lemma}\label{lemma:noisyfevalbound}
Let Assumption \ref{assumption:GP} hold, and suppose the stationary kernel is normalized so
that $k(\theta,\theta)=1$ for all $\theta\in\Theta$. Let $\sigma>0$ and
$m\in\mathbb N$. Suppose also that, in local coordinates around the point of interest, there exists $\rho>0$ such that $\{\pm h e_j:0<h\leq \rho,\, j\in[d]\}\subset \Theta$. Then
\[
\min_{\mathbf z:|\mathbf z|=2md}
\operatorname{Tr}\!\left(
\nabla k_{\mathbf z,\sigma^2}(\mathbf 0,\mathbf 0)\nabla^\top
\right)
\le
C\sigma d m^{-1/2},
\]
where $C>0$ depends only on the kernel and the local neighborhood, but not on
$d$, $m$, or $\sigma$.
\end{lemma}

\begin{proof}
Here $z$ is understood as a finite design, possibly with repeated points (so it is not a set).
Let
$
\gamma:=\frac{\sigma^2}{m}.
$
By Lemma \ref{lemma:forwarddiff}, for any admissible $h>0$ there exists a design $D$ of size $2md$ such that
\[
\operatorname{Tr}\!\left(
\nabla k_{D,\sigma^2}(\mathbf 0,\mathbf 0)\nabla^\top
\right)
\le
\sum_{j=1}^d A_j(h,\gamma),
\]
where
\[
A_j(h,\gamma)
:=
-\partial_j^2\phi(0)
-
\frac{2(\partial_j\phi(-he_j))^2}
{1-\phi(2he_j)+\gamma}.
\]
Thus it suffices to show that, for a suitable choice of $h$,
\[
A_j(h,\gamma)\le C\sqrt{\gamma} = C\sigma m^{-1/2}
\]
uniformly over $j\in[d]$. Fix $j\in[d]$ and define 
$
\psi(s):=\phi(se_j).
$
Since $k$ is stationary and isotropic, $\psi$ does not depend on the choice of $j$. Moreover, $\psi$ is even and belongs to $C^4$ by Assumption \ref{assumption:GP}. Moreover, since we assume the kernel is normalized, we have $\psi(0)=1$.

Now let
$
a:=-\psi''(0).
$
By Bochner's theorem (e.g. Theorem 6.6 in \citet{wendland2004scattered}), a stationary positive
definite kernel admits the representation
\[
\phi(t)=\int_{\mathbb R^d} e^{i\langle \omega,t\rangle}\,d\nu(\omega)
\]
for a finite nonnegative spectral measure $\nu$. Since $\phi \in C^2$ at the origin, the spectral measure $\nu$ has finite coordinate second moments, and
\[
\partial_j^2\phi(0)
=
-\int_{\mathbb R^d}\omega_j^2\,d\nu(\omega).
\] Moreover, since $k$ is nonconstant, we know that $\nu\Big\{\mathbb{R}^d \backslash \{\mathbf{0}\} \Big\}>0$, so that $\int_{\mathbb{R}^d} \|\omega\|^2 d\nu(\omega)>0$, i.e. $\nu$ is not concentrated at the origin. Finally, by our assumed isotropy of $k$, $\nu$ is rotationally invariant so that $\int_{\mathbb{R}^d}\omega_j^2 d\nu(\omega) = \frac{1}{d} \int_{\mathbb{R}^d} \|\omega\|^2 d\nu(\omega)>0$ for every $j \in [d]$. Hence
\[
\partial_j^2\phi(0)
=
-\int_{\mathbb R^d}\omega_j^2\,d\nu(\omega)
<0,
\]
so $a=-\psi''(0)>0$. Now, we have
\[
\partial_j^2\phi(0)=\psi''(0)=-a, \qquad \partial_j\phi(-he_j)=\psi'(-h),
\]
and because $\psi$ is even, we have
\[
\psi'(0)=0, \qquad \psi^{(3)}(0)=0.
\]
Therefore, as $h\to 0$, a Taylor expansion yields
\[
\psi(2h) = 1-2ah^2+O(h^4),
\]
and
\[
\psi'(-h) = ah+O(h^3).
\]
Therefore,
\[
1-\phi(2he_j)+\gamma = 1-\psi(2h)+\gamma = 2ah^2+\gamma+O(h^4),
\]
and
\[
2(\partial_j\phi(-he_j))^2 = 2(\psi'(-h))^2
= 2a^2h^2+O(h^4).
\]
Plugging these expansions into $A_j(h,\gamma)$ gives
\[
\begin{aligned}
A_j(h,\gamma)
&=
a- \frac{2a^2h^2+O(h^4)} {2ah^2+\gamma+O(h^4)}  \\
&=
\frac{a\gamma+O(h^4)} {2ah^2+\gamma+O(h^4)}.
\end{aligned}
\]

Now we may choose $h_0>0$ small enough that $h_0\le \rho$ and, for all
$0<h\le h_0$,
\[
2ah^2+\gamma+O(h^4)\ge ah^2+\gamma,
\]
which we can do because the $O(h^4)$ term is dominated by $ah^2$ for small enough $h$. To conclude, first suppose that $\gamma\leq h_0^4$. In this case set
$
h=\gamma^{1/4}.
$
Then $h\leq h_0$, $h^2=\sqrt{\gamma}$, and $h^4=\gamma$. Hence
\[
A_j(h,\gamma) \leq C\frac{\gamma}{\sqrt{\gamma}}
=
C\sqrt{\gamma}.
\]
All that remains to handle is the case $\gamma>h_0^4$. In this case, choose the fixed admissible value $h=h_0$. Since the second term in the definition of
$A_j(h,\gamma)$ is nonnegative, we have
$
A_j(h_0,\gamma)\leq a.
$
But $\gamma>h_0^4$ implies $\sqrt{\gamma}>h_0^2$, and, rearranging this inequality and multiplying both sides by $a$, we have
$a\leq \frac{a}{h_0^2}\sqrt{\gamma}.$
Thus, after increasing the constant if necessary, we obtain
\[
A_j(h,\gamma)\le C\sqrt{\gamma}
\]
for all $\gamma>0$. Therefore, for every $j\in[d]$,
\[
A_j(h,\gamma) \leq C\sqrt{\gamma} = C\sigma m^{-1/2}.
\]
Finally, summing over $[d]$ and applying Lemma \ref{lemma:forwarddiff} yields
\[
\min_{z:|z|=2md}
\operatorname{Tr}\!\left(\nabla k_{z,\sigma^2}(0,0)\nabla^\top\right) \leq C d\sqrt{\gamma} = C\sigma d m^{-1/2}.
\]
\end{proof}
\begin{lemma}
\label{lemma:stayinLSC}
Let $f_{\mathbf x}$ be twice continuously differentiable on $\Theta$. Suppose that
$f_{\mathbf x}$ is $L$-smooth on $\Theta$, and that $f_{\mathbf x}$ is $\tau$-locally strongly convex on $ B_r(\theta^*) \subset \Theta$.
Assume also that $\theta^{(0)}\in  B_r(\theta^*)$ and
$\nabla f_{\mathbf x}(\theta^*)=0$. 
Moreover, let the iterates satisfy
\[
\theta^{(t+1)} = \theta^{(t)} - \eta\left( \nabla f_{\mathbf x}(\theta^{(t)})+b_t+n_t \right),
\]
where $\eta\leq 1/L$. Assume that with probability at least $1-\delta$, simultaneously for all $t\leq T$,
\[
\|b_t\|\le \frac{\tau r}{2},
\qquad
\|n_t\|\le \frac{\tau r}{2}.
\]
Then, with probability at least $1-\delta$,
\[
\theta^{(t)}\in  B_r(\theta^*)
\qquad
\forall t\le T.
\]
\end{lemma}

\begin{proof}
We do our analysis on the event that the two bounds
\[
\|b_t\|\le \frac{\tau r}{2},
\qquad
\|n_t\|\le \frac{\tau r}{2}
\]
hold for all $t\le T$. We prove the claim by induction.

The base case holds by assumption, since
$\theta^{(0)}\in  B_r(\theta^*)$. Now assume
$\theta^{(t)}\in  B_r(\theta^*)$. Because
$ B_r(\theta^*)$ is convex and contained in $\Theta$, the full segment
$
\theta^*+s(\theta^{(t)}-\theta^*),
$ is contained in $\Theta$ for all $s \in [0,1]$. Thus we can make use of $L$-smoothness and $\tau$-local strong convexity on this segment. Since $\nabla f_{\mathbf x}(\theta^*)=0$, the fundamental theorem of calculus
gives
\[
\nabla f_{\mathbf x}(\theta^{(t)}) = \int_0^1 \nabla^2 f_{\mathbf x} \left( \theta^*+s(\theta^{(t)}-\theta^*) \right) (\theta^{(t)}-\theta^*)\,ds.
\]
Now define
\[
H_t := \int_0^1 \nabla^2 f_{\mathbf x} \left( \theta^*+s(\theta^{(t)}-\theta^*) \right)\,ds,
\]
so that
\[
\nabla f_{\mathbf x}(\theta^{(t)}) = H_t(\theta^{(t)}-\theta^*).
\]
By $L$-smoothness and $\tau$-local strong convexity on
$ B_r(\theta^*)$, every Hessian on the segment has eigenvalues in $[\tau,L]$. Therefore $H_t$ also has eigenvalues in $[\tau,L]$. Hence, since $\eta\leq 1/L$,
\[
\|I-\eta H_t\| \leq 1-\eta\tau.
\]
Now, writing out our gradient update,
\[
\begin{aligned}
\|\theta^{(t+1)}-\theta^*\|
&=
\left\| (I-\eta H_t)(\theta^{(t)}-\theta^*) - \eta(b_t+n_t) \right\| \\
&\leq (1-\eta\tau)\|\theta^{(t)}-\theta^*\| + \eta\|b_t+n_t\| \\
&\leq (1-\eta\tau)r + \eta\left(\|b_t\|+\|n_t\|\right) \\
&\leq (1-\eta\tau)r + \eta\tau r \\
&=
r.
\end{aligned}
\]
Therefore $\theta^{(t+1)}\in  B_r(\theta^*)$. By induction, we conclude that $\theta^{(t)}\in  B_r(\theta^*)$ for all $t\leq T$.
\end{proof}

\begin{lemma}
\label{lemma:gradientdescentanalysis}
Admit the assumptions and conditions of Lemma~\ref{lemma:stayinLSC}. Suppose
that, on the same high probability event, for all $t\in[T]$,
\[
\|b_t\|\leq C_t \leq \frac{\tau r}{2},
\]
and
\[
\|n_t\|\leq
N_\delta := \frac{2B\sqrt T}{n\mu} \left( 4\sqrt d+2\sqrt{2\log\frac{T}{\delta}}
\right).
\]
Then, with probability at least $1-\delta$, for each $t\in[T]$,
\[
f_{\mathbf x}(\theta^{(t+1)}) \leq f_{\mathbf x}(\theta^{(t)}) - \frac{\eta}{2} \|\nabla f_{\mathbf x}(\theta^{(t)})\|^2 + \frac{\eta}{2}C_t^2 + \eta\mathcal V,
\]
where
$
\mathcal V
:=
\frac{\tau r}{2}N_\delta+\frac12N_\delta^2.
$
In particular, if
\[
n = \Omega\left( \frac{\sqrt{T(d+\log(T/\delta))}}{\mu} \right),
\]
then
\[
\mathcal V = O\left( \frac{\sqrt{T(d+\log(T/\delta))}}{n\mu}
\right).
\]
\end{lemma}
\begin{proof}
We work on the high probability event from Lemma~\ref{lemma:stayinLSC} on which
the iterates (and all line segments between them) remain in the region where $f_{\mathbf x}$ is $L$-smooth and the bounds on $b_t$ and $n_t$ hold simultaneously for all $t\in[T]$.
Fix $t\in[T]$ and write
\[
g_t:=\nabla f_{\mathbf x}(\theta^{(t)}),
\qquad
e_t:=b_t+n_t.
\]
By the update rule,
\[
\theta^{(t+1)}
=
\theta^{(t)}-\eta(g_t+e_t).
\]
Since $f_{\mathbf x}$ is $L$-smooth along the segment joining
$\theta^{(t)}$ and $\theta^{(t+1)}$, we have
\[
f_{\mathbf x}(\theta^{(t+1)})
\le
f_{\mathbf x}(\theta^{(t)})
+
\langle g_t,\theta^{(t+1)}-\theta^{(t)}\rangle
+
\frac{L}{2}
\|\theta^{(t+1)}-\theta^{(t)}\|^2.
\]
Substituting the update gives
\[
\begin{aligned}
f_{\mathbf x}(\theta^{(t+1)})
&\le
f_{\mathbf x}(\theta^{(t)})
-
\eta\langle g_t,g_t+e_t\rangle
+
\frac{L\eta^2}{2}
\|g_t+e_t\|^2.
\end{aligned}
\]
Since $\eta\le 1/L$, we have
$
\frac{L\eta^2}{2}\le \frac{\eta}{2},
$
and hence,
\[
\begin{aligned}
f_{\mathbf x}(\theta^{(t+1)})
&\le
f_{\mathbf x}(\theta^{(t)})
-
\eta\langle g_t,g_t+e_t\rangle
+
\frac{\eta}{2}
\|g_t+e_t\|^2 \\
&=
f_{\mathbf x}(\theta^{(t)})
-
\frac{\eta}{2}\|g_t\|^2
+
\frac{\eta}{2}\|e_t\|^2.
\end{aligned}
\]
Thus,
\[
f_{\mathbf x}(\theta^{(t+1)})
\le
f_{\mathbf x}(\theta^{(t)})
-
\frac{\eta}{2}\|\nabla f_{\mathbf x}(\theta^{(t)})\|^2
+
\frac{\eta}{2}\|b_t+n_t\|^2.
\]
Now,
\[
\frac{\eta}{2}\|b_t+n_t\|^2
\le
\frac{\eta}{2}(C_t+N_\delta)^2
=
\frac{\eta}{2}C_t^2
+
\eta C_tN_\delta
+
\frac{\eta}{2}N_\delta^2,
\]
so that using $C_t\le \frac{\tau r}{2}$,
\[
\frac{\eta}{2}\|b_t+n_t\|^2
\le
\frac{\eta}{2}C_t^2
+
\eta
\left(
\frac{\tau r}{2}N_\delta+\frac12N_\delta^2
\right).
\]
Therefore, with
\[
\mathcal V
:=
\frac{\tau r}{2}N_\delta+\frac12N_\delta^2,
\]
we obtain
\[
f_{\mathbf x}(\theta^{(t+1)})
\le
f_{\mathbf x}(\theta^{(t)})
-
\frac{\eta}{2}\|\nabla f_{\mathbf x}(\theta^{(t)})\|^2
+
\frac{\eta}{2}C_t^2
+
\eta\mathcal V.
\]
Finally,
\[
N_\delta
=
\frac{2B\sqrt T}{n\mu}
\left(
4\sqrt d+2\sqrt{2\log\frac{T}{\delta}}
\right)
=
O\left(
\frac{\sqrt T\sqrt{d+\log(T/\delta)}}{n\mu}
\right).
\]
Under the stated lower bound on $n$, $N_\delta=O(1)$, so
$
N_\delta^2=O(N_\delta).
$
Hence, we conclude
\[
\mathcal V
=
\frac{\tau r}{2}N_\delta+\frac12N_\delta^2
=
O\left(
\frac{\sqrt {T(d+\log(T/\delta))}}{n\mu}
\right).
\]
\end{proof}
\section{Proof of Privacy Result}
\label{appendix:proofs}
\subsection*{Proof of Theorem \ref{thm:privacyguarantee}}
\begin{proof}
    The result follows a standard argument that combines guarantees of the Gaussian mechanism, post-processing and composition.  Let $\mu > 0$ be the desired level of privacy. We will first prove that each iteration of Algorithm \ref{algorithm} is $\frac{\mu}{\sqrt{T}}$-GDP, given the output of the previous iterations. Then our result follows by composition. To that end, take two datasets $x$ and $x'$ such that $d_H(x, x') = 1$. Without loss of generality, we can assume that $x_1 \neq x_1'$ and $x_i = x_i'$ for all $2 \leq i \leq n$. Fix $t > 0$, the current iteration of the algorithm. First note that given the output of the previous iterations, the only further privacy leakage in iteration $t$ is through the computation of the approximate gradients. Indeed, lines 4 and 5 are functions exclusively of the kernel $k$, the evaluation set of the previous iteration $\mathcal D_{t-1}$, and the parameter output of the previous iteration $\theta^{(t)}$. By the post-processing property of Gaussian Differential Privacy, these contribute no further privacy leakage (Proposition 4 of \citet{dong2022gaussian}). Thus, given the output of the previous iterations, the only new user data entering our output at iteration $t$ occurs in line $9$, through $g_t$. Hence, it suffices to show that given $\theta^{(t)}$ and $\mathcal{D}_{t-1}$, line 9 is $\frac{\mu}{\sqrt{T}}$-GDP. To that end, let $\widetilde{\mathbf{y}}_i^{(t)} := [(\mathbf{y}_i^{(0)})^\top, \ldots, (\mathbf{y}_i^{(t)})^\top]^\top$ 
    \[
    g_t^{(i)} := \nabla k(\theta^{(t)}, \mathcal{D})\big(k(\mathcal{D}, \mathcal{D}) + \sigma^2 I\big)^{-1} \widetilde{\mathbf{y}}_i^{(t)},
    \]
    and 
    \[
    \tilde{g}_t^{(i)} := \nabla k(\theta^{(t)}, \mathcal{D})\big(k(\mathcal{D}, \mathcal{D}) + \sigma^2 I\big)^{-1} {\widetilde{\mathbf{y}}_i^{(t)'}}.
    \]
    Our estimated gradients at the $t$-th iteration for the two datasets are as follows
    \[
    g_t =\frac{1}{n} \sum_{i = 1}^n g_t^{(i)} \cdot \min \left \{ 1, \frac{B}{\|g_t^{(i)}\|} \right \},
    \]
    and 
    \[
    \Tilde{g}_t = \frac{1}{n} \sum_{i = 1}^n \Tilde{g}_t^{(i)} \cdot \min \left \{ 1, \frac{B}{\|\Tilde{g}_t^{(i)}\|} \right \}.
    \]
    Hence, by noting that $g_t^{(i)} = \Tilde{g}_t^{(i)}$ for all $2 \leq i \leq n$, we derive that the global sensitivity of the estimated gradient at time $t$ is
    \begin{align*}
        \|g_t - \Tilde{g}_t \| &= \left \| \frac{1}{n}\sum_{i = 1}^n g_t^{(i)} \cdot \min \left \{ 1, \frac{B}{\|g_t^{(i)}\|} \right \} - \frac{1}{n} \sum_{i = 1}^n \Tilde{g}_t^{(i)} \cdot \min \left \{ 1, \frac{B}{\|\Tilde{g}_t^{(i)}\|} \right \} \right \| \\
        &= \frac{1}{n} \left \| g_t^{(1)} \cdot \min \left \{ 1, \frac{B}{\|g_t^{(1)}\|} \right \} -  \Tilde{g}_t^{(1)} \cdot \min \left \{ 1, \frac{B}{\|\Tilde{g}_t^{(1)}\|} \right \}  \right \| \\
        &\leq \frac{1}{n} \left ( \left \| g_t^{(i)} \cdot \min \left \{ 1, \frac{B}{\|g_t^{(i)}\|} \right \} \right \| + \left \| \Tilde{g}_t^{(1)} \cdot \min \left \{ 1, \frac{B}{\|\Tilde{g}_t^{(1)}\|} \right \} \right \| \right ) \\
        &\leq \frac{1}{n} (B + B) = \frac{2B}{n}.
    \end{align*}
    
Hence, since we are adding $\frac{2B\sqrt{T}}{n\mu}w_t$, where $\ w_t \overset{iid}{\sim} N(0, I)$, by Theorem \ref{gaussianmechanism}, every iteration of the algorithm is $\frac{\mu}{\sqrt{T}}$-GDP, given the output of the previous algorithms. Finally, then, since our whole algorithm is a T-fold composition of $\frac{\mu}{\sqrt{T}}$-GDP mechanisms, we conclude that Algorithm \ref{algorithm} is $\mu$-GDP by Corollary \ref{composition}.
\end{proof}

\section{Proofs of Section \ref{section:noisy}}
\subsection*{Proof of Lemma \ref{lemma:placeholder}}
\begin{proof}
    By assumption, $\mathcal{L}(\cdot, x_i) \sim \mathcal{GP}(\mathbf{0},k)$ with $k \in C^{4+\alpha}$ stationary. As a result, $\mathcal{L}(\cdot, x_i)$ admits a modification (which we continue to denote by $\mathcal{L}(\cdot, x_i)$ for simplicity) whose sample paths lie almost surely in $C^2$ \citep{da2026sample}. In particular, since $\Theta$ is compact, the random variables 
    \[
    B_i := \sup_{\theta \in \Theta} \|\nabla \mathcal{L}(\theta, x_i)\|, \quad H_i := \sup_{\theta \in \Theta} \|\nabla^2 \mathcal{L}(\theta, x_i)\|
    \]
    are almost surely finite. Moreover, since we can write $B_i = \sup_{\theta \in \Theta, u \in \mathbb{S}^{d-1}} u^\top \nabla \mathcal{L}(\theta, x_i)$ and $H_i = \sup_{\theta \in \Theta, u, v \in \mathbb{S}^{d-1}} u^\top \nabla^2 \mathcal{L}(\theta, x_i) v$, it follows that these are suprema of centered Gaussian processes. Since $\Theta \times \mathbb S^{d-1}$ and $\theta \times \mathbb S^{d-1} \times \mathbb S^{d-1}$ are compact index sets and the corresponding processes have continuous sample paths, we can apply Borell-TIS \citep{adler2007random}. It follows that the constants
    \[
    m_B := \mathbb{E}[\sup_{\theta \in \Theta} \|\nabla \mathcal{L}(\theta, x_1)\|], \quad \sigma_B^2 := \sup_{\theta \in \Theta, u \in \mathbb S^{d-1}} \text{Var}\big(u^\top \nabla \mathcal{L}(\theta, x_1)\big)
    \]
    and
    \[
    m_H :=  \mathbb{E}[\sup_{\theta \in \Theta} \|\nabla^2 \mathcal{L}(\theta, x_1)\|], \quad \sigma_H^2 := \sup_{\theta \in \Theta, u,v \in \mathbb S^{d-1}} \text{Var}\big(u^\top \nabla^2 \mathcal{L}(\theta, x_1)v \big)
    \]
    are finite by Fernique's Theorem \citep{fernique2006regularite}. By Borell-TIS, it then follows that
    \begin{align}\label{eq:highprob1}
    \mathbb{P}\Big(\max_{i \in [n]} B_i > m_B + \sigma_B\sqrt{2 \log \frac{2n}{\delta}} \Big) \leq \frac{\delta}{2}
    \end{align}
    and
    \begin{align}\label{eq:highprob2}
    \mathbb{P}\Big(\max_{i \in [n]} H_i > m_H + \sigma_H\sqrt{2 \log \frac{2n}{\delta}} \Big) \leq \frac{\delta}{2}.
    \end{align}
    Thus, take $B := m_B + \sigma_B\sqrt{2 \log \frac{2n}{\delta}} $ and $L :=  m_H + \sigma_H\sqrt{2 \log \frac{2n}{\delta}}$. Then, on the event of the intersection of \eqref{eq:highprob1}-\eqref{eq:highprob2}, which holds with probability at least $1-\delta$, we have \[
    \max_{i \in [n]} \sup_{\theta \in \Theta} \|\nabla \mathcal{L}(\theta, x_i)\| \leq B.
    \]
    Also on this event, we have
    \[
    \sup_{ \theta \in \Theta} \|\nabla^2 f_\mathbf{x}(\theta) \| \leq \frac{1}{n} \sum_{i=1}^n \sup_{\theta \in \Theta} \|\nabla^2 \mathcal{L}(\theta, x_i)\|  \leq L, 
    \]
    and now, using convexity of $\Theta$, for $\theta, \theta' \in \Theta$,
    \[
    \nabla f_\mathbf{x}(\theta) - \nabla f_\mathbf{x}(\theta')= \int_0^1 \nabla^2 f_\mathbf{x}(\theta' + t(\theta - \theta'))(\theta - \theta')dt,
    \]
    so that
    \[
    \|\nabla f_\mathbf{x}(\theta) - \nabla f_\mathbf{x}(\theta')\|\leq \|\theta - \theta'\|\int_0^1 \|\nabla^2 f_\mathbf{x}(\theta' + t(\theta - \theta'))\|dt \leq L\|\theta - \theta'\|.
    \]
\end{proof}
\subsection*{Proof of Lemma \ref{lemma:placeh}}
\begin{proof}
By the proof of Lemma \ref{lemma:placeholder}, we know that $f_\mathbf{x}$ has a $C^2$ modification. Moreover, by Bulinskaya's lemma, $f_\mathbf{x}$ is almost surely Morse on $\Theta$ \cite{adler2007random}. Consequently, we have almost surely  $\lambda^* := \lambda_{\min}(\nabla^2 f_\mathbf{x}(\theta^*)) >0$. Then, for $\theta^* \in \text{int}(\Theta)$, almost sure continuity of $\nabla^2 f_\mathbf{x}$ implies that the random variable
\[
R_* := \sup\{s>0: B_s(\theta^*) \subseteq \Theta \text{ and } \lambda_{\min}(\nabla^2f_\mathbf{x}(z)) \geq \frac{\lambda_*}{2} \text{ for all }z \in B_s(\theta^*) \}
\]
satisfies $R_* >0$ almost surely. Because $\lambda_*, R_*>0$ almost surely, we can choose deterministic constants $\tau, r>0$ such that
\[
\mathbb{P}(\lambda_* < 2\tau) \leq \frac{\delta}{2} \quad \text{ and } \quad \mathbb{P}(R_* \leq r) \leq \frac{\delta}{2},
\]
and on the event $E := \{ \lambda_* \geq 2\tau\} \cap \{R_* >r\}$, for every $z \in B_r(\theta^*)$,
\[
\lambda_{\min}(\nabla^2 f_\mathbf{x}(z)) \geq \frac{\lambda_*}{2} \geq \tau.
\]
Hence, $\nabla^2 f_\mathbf{x}(z) \succeq \tau I $ for all $z \in B_r(\theta^*)$. Now fix $\theta, \theta' \in B_r(\theta^*)$. By convexity,
\[
\nabla f_\mathbf{x}(\theta) - \nabla f_\mathbf{x}(\theta') = \int_0^1 \nabla^2 f_\mathbf{x}(\theta' + t(\theta - \theta')) (\theta - \theta')dt,
\]
and hence
\begin{align*}
\langle\nabla f_\mathbf{x}(\theta) - \nabla f_\mathbf{x}(\theta'), \theta - \theta'\rangle &= \int_0^1 \langle\nabla^2 f_\mathbf{x}(\theta' + t(\theta - \theta')) (\theta - \theta'), \theta - \theta'\rangle dt \\
&= \int_0^1 (\theta - \theta')^\top \nabla^2 f_\mathbf{x}(\theta' + t(\theta - \theta')) (\theta - \theta')dt \\
&\geq \int_0^1\tau \|\theta - \theta'\|^2dt = \tau\|\theta - \theta'\|^2.
\end{align*}
Finally, to conclude, since $\theta^* \in \text{int}(\Theta)$, it follows that we can shrink $r$ to ensure $B_{r}(\theta^*) \subset \Theta$, proving the claim.
\end{proof}
\subsection*{Proof of Theorem \ref{theorem:functionEvalsNoisy}}
\begin{proof}
For ease of notation, we write
\[
L:=L(\delta),\qquad \tau:=\tau(\delta),\qquad r:=r(\delta).
\]
We first prove the stated suboptimality bound for \(T\asymp \log n\), and
then bound the number of function evaluations.

To that end, let \(\mathcal E_{\mathrm{reg}}\) denote the event on which \(f_{\mathbf x}\) is twice continuously differentiable on the relevant neighborhood of \(\Theta\), is \(L\)-smooth on \(\Theta\), is \(\tau\)-locally strongly convex on \(B_r(\theta^*)\), and satisfies the gradient boundedness
condition used by the clipping step. By Lemma~\ref{lemma:placeholder},
\[
\mathbb P(\mathcal E_{\mathrm{reg}})\ge 1-\delta.
\]
Next, let \(\mathcal E_{\mathrm{alg}}\) denote the event on which the GP
gradient approximation error and the privacy-noise bound hold simultaneously
for all \(t\le T\). By construction of Line 4 of the algorithm,
\[
\operatorname{Tr}\!\left(
\nabla k_{\mathcal D_t,\sigma^2}
(\theta^{(t)},\theta^{(t)})
\nabla^\top
\right)
\le \varepsilon
\qquad
\forall t\le T.
\]
Therefore, by Lemma~\ref{lemma:gpboundappendix}, with probability at least
\(1-\delta/2\),
\[
\left\|
g_t-\nabla f_{\mathbf x}(\theta^{(t)})
\right\|^2
\le
5\varepsilon\log\frac{2nT}{\delta}
\qquad
\forall t\le T.
\]
Now, as before, define the gradient approximation bias
$
b_t:=g_t-\nabla f_{\mathbf x}(\theta^{(t)}).
$
Then, on this event,
\[
\|b_t\|
\le
C_t
:=
\sqrt{5\varepsilon\log\frac{2nT}{\delta}}
\qquad
\forall t\le T.
\]
By the assumed upper bound on \(\varepsilon\), we have
$
C_t\le \frac{\tau r}{2}.
$
Similarly, by Lemma~\ref{subGaussianProof}, with probability at least
\(1-\delta/2\), the privacy noise satisfies
\[
\|n_t\|
\le
N_\delta
:=
\frac{2B\sqrt T}{n\mu}
\left(
4\sqrt d+2\sqrt{2\log\frac{2T}{\delta}}
\right)
\qquad
\forall t\le T.
\]
The assumed lower bound
\[
n=\Omega\!\left(
\frac{\sqrt{T(d+\log(T/\delta))}}{\mu}
\right)
\]
with a sufficiently large hidden constant implies
$
N_\delta\le \frac{\tau r}{2},
$
so we conclude
\[
\mathbb P(\mathcal E_{\mathrm{alg}})\ge 1-\delta.
\]

We now work on the event
$
\mathcal E:=\mathcal E_{\mathrm{reg}}\cap \mathcal E_{\mathrm{alg}},
$
which has probability at least \(1-2\delta\). On \(\mathcal E\), the
hypotheses of Lemma~\ref{lemma:stayinLSC} are satisfied. Since
\(\|\theta^{(0)}-\theta^*\|\le r\), Lemma~\ref{lemma:stayinLSC} implies
\[
\theta^{(t)}\in B_r(\theta^*)
\qquad
\forall t\le T,
\]
so that we can apply local strong convexity results. 
By Lemma~\ref{lemma:gradientdescentanalysis}, on \(\mathcal E\), for every
\(t\le T\),
\[
f_{\mathbf x}(\theta^{(t+1)})
\le
f_{\mathbf x}(\theta^{(t)})
-
\frac{\eta}{2}
\left\|
\nabla f_{\mathbf x}(\theta^{(t)})
\right\|^2
+
\frac{\eta}{2}C_t^2
+
\eta\mathcal V,
\]
where
\[
\mathcal V
=
O\!\left(
\frac{\sqrt T\sqrt{d+\log(T/\delta)}}{n\mu}
\right).
\]
Substituting the definition of \(C_t\) gives
\[
f_{\mathbf x}(\theta^{(t+1)})
\le
f_{\mathbf x}(\theta^{(t)})
-
\frac{\eta}{2}
\left\|
\nabla f_{\mathbf x}(\theta^{(t)})
\right\|^2
+
\frac{5\eta\varepsilon}{2}
\log\frac{2nT}{\delta}
+
\eta\mathcal V .
\]
Finally, let
\[
F_t:=f_{\mathbf x}(\theta^{(t)})-f_{\mathbf x}(\theta^*).
\]
Since the iterates remain in \(B_r(\theta^*)\), and \(f_{\mathbf x}\) is
\(\tau\)-locally strongly convex on this ball, the Polyak--Lojasiewicz
inequality implied by strong convexity gives
\[
\left\|
\nabla f_{\mathbf x}(\theta^{(t)})
\right\|^2
\ge
2\tau
\left(
f_{\mathbf x}(\theta^{(t)})-f_{\mathbf x}(\theta^*)
\right)
=
2\tau F_t.
\]
Therefore,
\[
\begin{aligned}
F_{t+1}
&\le
F_t
-
\eta\tau F_t
+
\frac{5\eta\varepsilon}{2}
\log\frac{2nT}{\delta}
+
\eta\mathcal V \\
&=
(1-\eta\tau)F_t
+
\frac{5\eta\varepsilon}{2}
\log\frac{2nT}{\delta}
+
\eta\mathcal V .
\end{aligned}
\]
Unrolling this recursion yields
\[
\begin{aligned}
F_T
&\le
(1-\eta\tau)^T F_0
+
\left(
\frac{5\eta\varepsilon}{2}
\log\frac{2nT}{\delta}
+
\eta\mathcal V
\right)
\sum_{s=0}^{T-1}(1-\eta\tau)^s \\
&\le
(1-\eta\tau)^T F_0
+
\frac{5\varepsilon}{2\tau}
\log\frac{2nT}{\delta}
+
\frac{\mathcal V}{\tau},
\end{aligned}
\]
where the final inequality uses a geometric series bound. 
Now choose the hidden constant in \(T\asymp \log n\) large enough so that
\[
(1-\eta\tau)^T F_0
=
O\!\left(
\frac{
\sqrt{\log n\left(d+\log\frac{\log n}{\delta}\right)}
}{
n\mu
}
\right),
\]
and since
\[
\mathcal V
=
O\!\left(
\frac{\sqrt T\sqrt{d+\log(T/\delta)}}{n\mu}
\right),
\]
taking \(T\asymp \log n\) gives
\[
\frac{\mathcal V}{\tau}
=
O\!\left(
\frac{
\sqrt{\log n\left(d+\log\frac{\log n}{\delta}\right)}
}{
n\mu
}
\right).
\]
Thus
\[
F_T
=
O\!\left(
\frac{
\sqrt{\log n\left(d+\log\frac{\log n}{\delta}\right)}
}{
n\mu
}
+
\varepsilon\log\frac{n\log n}{\delta}
\right).
\]

It remains to bound the number of function evaluations needed to achieve the
desired value of \(\varepsilon\). By Lemma~\ref{lemma:noisyfevalbound}, for
\(\omega\in\mathbb N\), one can achieve
\[
\varepsilon
=
O\!\left(d\sigma\omega^{-1/2}\right)
\]
using at most \(O(\omega d)\) function evaluations. By stationarity, the same
local design can be translated to the current iterate \(\theta^{(t)}\).
Moreover, posterior covariance is monotone under adding observations, so the
previously collected design \(\mathcal D_{t-1}\) can only reduce the posterior
gradient covariance. Hence each iteration requires at most \(O(\omega d)\)
new function evaluations.

Since \(T\asymp \log n\), the total number of function evaluations is
$
O(\omega d\log n).
$
Finally,
$
\log\frac{n\log n}{\delta}
=
O\!\left(\log\frac{n}{\delta}\right),
$
so substituting
$
\varepsilon
=
O\!\left(d\sigma\omega^{-1/2}\right)
$
gives
\[
F_T
=
O\!\left(
\frac{
\sqrt{\log n\left(d+\log\frac{\log n}{\delta}\right)}
}{
n\mu
}
+
\frac{\sigma d\log\frac{n}{\delta}}{\sqrt{\omega}}
\right).
\]
This holds on \(\mathcal E\), which has probability at least \(1-2\delta\).
\end{proof}

\section{Additional Results}
\subsection{Convergence Guarantees in the Noiseless Case}
\label{section:convergencenoiseless}
In this section we present the theoretical guarantees for the estimates obtained by Algorithm \ref{algorithm}, if $\sigma^2 =0$. To approximate our gradient, we use minimum norm Reproducing Kernel Hilbert Space interpolation, which requires our function to be an element of $k$'s RKHS.
\begin{assumption}
    \label{assumption:rkhs}
    For all $i \in [n]$, $\mathcal{L}(\cdot, x_i) \in \mathcal{H}:= RKHS(k)$, where $k$ is the kernel used in Algorithm \ref{algorithm}. Moreover, there exists a constant $C_\mathcal{X} < \infty$ such that for all $i \in [n]$, 
    $
    \|\mathcal{L}(\cdot, x_i)\|_\mathcal{H} \leq C_\mathcal{X}
    $
\end{assumption}
Since now the validation functions are deterministic, we assume all ingredients we need for our analysis:
\begin{assumption}\label{assumptions:noiselessall}
    The function $f_\mathbf{x}$ is $L$-smooth on $\Theta$ and the function $f_\mathbf{x}$ has a nondegenerate local minimum $\theta^* \in \text{int}(\Theta)$ such that $\|\theta^{(0)} - \theta^*\| \leq r$, and $f_\mathbf{x}(\theta)$ is $\tau$-LSC in $B_r(\theta^*)$. Finally, $\|\nabla \mathcal{L}(\theta, x_i)\| \leq B$ for all $i \in [n]$, $\theta \in \Theta$. 
\end{assumption}

We are now ready to state Algorithm \ref{algorithm}'s convergence guarantees in the case of $\sigma^2 = 0$.
The following theorem provides a high probability bound on the suboptimality gap $F_T$, at the conclusion of the Algorithm's $T$ iterations.
\begin{theorem}
    \label{thm:Goodconvergence}
    Under Assumptions \ref{assumption:rkhs}-\ref{assumptions:noiselessall}, and if $\varepsilon \leq \frac{\tau^2r^2}{4C_\mathcal{X}^2}$, and $n = \Omega(\frac{\sqrt{T(d + \log\frac{T}{\delta})}}{\mu})$, we have with probability at least $1-\delta$, for $\eta \leq \frac{1}{L}$,
    \begin{align*}
        F_T \leq (1 - \eta \tau)^T F_0 + \frac{\varepsilon C_\mathcal{X}^2}{2\tau} + \frac{\mathcal{V}}{\tau},
    \end{align*}
    where \[
    \mathcal V = O\left(\frac{\sqrt{T(d + \log \frac{T}{\delta})}}{n\mu}\right)
    \]
\end{theorem}
\begin{proof}
    By Lemma 1 in \citet{wu2024behavior}, we have
    \begin{align}\label{eq:rkhsboundap}
    \|\nabla \mathcal L(\theta, x_i) - \nabla k(\theta, \mathcal D)k(\mathcal D, \mathcal D)^{-1} \mathbf{y}_i\|^2 \leq \|f\|_\mathcal H^2 \text{Tr}(\nabla k_\mathcal D(\theta, \theta)\nabla^\top).
    \end{align}
    To conclude our proof, 
     we first prove the following intermediate lemma.
    \begin{lemma}
        Let $g_t$ be as defined in Algorithm \ref{algorithm}. Under Assumptions \ref{assumption:rkhs} and \ref{assumptions:noiselessall} and for $\sigma =0$, we have, for all $i \in [n]$, $t \in [T]$,
    \[
    \left \| g_t - \nabla f_\mathbf{x}(\theta^{(t)})  \right \| \leq C_\mathcal{X} \sqrt{\text{Tr}(\nabla k_{\mathcal{D}_t} (\theta^{(t)}, \theta^{(t)})\nabla^\top)}
    \]
    \end{lemma}
    \begin{proof}
         We start by using the triangle inequality to go to the individual gradients, then use Lemma \ref{lem:projection} to remove the projection operator, and lastly apply \eqref{eq:rkhsboundap}:
    \begin{align*}
        \left \| g_t - \nabla f_\mathbf{x}(\theta^{(t)}) \right \| &= \left \| \frac{1}{n} \sum_{i=1}^n \Pi_B \left ( g_t^{(i)}(\theta^{(t)}) \right ) - \frac{1}{n} \sum_{i=1}^n \nabla \mathcal{L}(\theta^{(t)}, x_i)  \right \| \\
        & \leq \frac{1}{n} \sum_{i=1}^n \left \| \Pi_B \left ( g_t^{(i)}(\theta^{(t)}) \right ) - \nabla \mathcal{L}(\theta^{(t)}, x_i)  \right \| \\
        & \leq \frac{1}{n} \sum_{i=1}^n \left \| g_t^{(i)}(\theta^{(t)}) - \nabla \mathcal{L}(\theta^{(t)}, x_i)  \right \| \\
        & \leq \frac{1}{n} \sum_{i=1}^n C_\mathcal{X} \sqrt{\text{Tr}(\nabla k_{\mathcal{D}_t} (\theta^{(t)}, \theta^{(t)})\nabla^\top)} \\
        &= C_\mathcal{X} \sqrt{\text{Tr}(\nabla k_{\mathcal{D}_t} (\theta^{(t)}, \theta^{(t)})\nabla^\top)}.
    \end{align*}
    \end{proof}
    The proof now follows almost identically to the proof of Theorem \ref{theorem:functionEvalsNoisy}, with the difference being that $C_t = C_\mathcal{X} \sqrt \varepsilon$. We conclude, following the same steps, that 
\begin{align*}
F_T
&\le
(1-\eta\tau)^T F_0
+
\left(
\frac{\eta\varepsilon C_\mathcal{X}^2}{2}
+
\eta\mathcal V
\right)
\sum_{s=0}^{T-1}(1-\eta\tau)^s \\
&\le
(1-\eta\tau)^T F_0
+
\frac{\varepsilon C_\mathcal{X}^2}{2\tau}
+
\frac{\mathcal V}{\tau},
\end{align*}
with $\mathcal{V} = O\Big(\frac{\sqrt{T(d + \log \frac{T}{\delta})}}{n\mu} \Big)$ as in Theorem \ref{theorem:functionEvalsNoisy}.
\end{proof}

Furthermore, we have the following bound on the number of function evaluations required to achieve the minimal achievable suboptimality gap under $\mu$-GDP by Theorem \ref{thm:Goodconvergence}.
\begin{lemma}
    \label{lem: noiselessfunctionevalbound}
    Let the Assumptions and conditions of Theorem \ref{thm:Goodconvergence} hold. Then, to achieve with probability at least $1-\delta$ the suboptimality gap
    \[
    F_T = O\left( \frac{\sqrt{\log n({d + \log\frac{\log n}{\delta})}} }{n\mu}  \right) =: \mathcal{E}_{priv},
    \]
    $O(d \log n)$ function evaluations suffice.
\end{lemma}
\begin{proof}
    This is a simple extension of Theorem \ref{thm:Goodconvergence}. By Lemma 2 of \citet{wu2024behavior}, $\mathcal{D}$ of size $d+1$ suffices to obtain
    \[
    \text{Tr}(\nabla k_\mathcal{D}(\mathbf{0}, \mathbf{0}) \nabla^\top) = 0.
    \]
    By stationarity of $k$, and the fact that posterior covariance is monotone under adding observations, it follows that in each iteration at most $d+1$ function evaluations suffice to achieve $\text{Tr}(\nabla k_\mathcal{D}(\mathbf{0}, \mathbf{0}) \nabla^\top) = 0$. Hence, taking the bound of Theorem \ref{thm:Goodconvergence} and plugging in $\varepsilon = 0$ and $T \asymp \log n$, we obtain
    \[
    F_T = O\left( \frac{\sqrt{\log n(d + \log \frac{\log n}{\delta})}}{n\mu} \right).
    \]
    The total number of required function evaluations in this case is $(d+1)T \sim (d+1) \log n$, such that at most $O(d \log n)$ function evaluations suffice to achieve the stated suboptimality gap.
\end{proof}
Here, $\mathcal{E}_{priv}$ is a privacy-induced floor on the suboptimality gap achievable through gradient descent methods, and is also found in other analyses of private gradient descent such as \citet{avella2023differentially}.

\section{Examples}
\label{appendix:examples}
In this section, we provide details for replicating the results of examples from the main manuscript. Code to replicate these examples can be found at

\begin{center}
\url{https://anonymous.4open.science/r/DPGIBO-9E44/}
\end{center}

In addition to this, we present two additional experiments. Further comparisons between (non-private) GIBO and global approaches can be found in \citet{muller2021local} and \citet{wu2024behavior}.

\subsection{Hyperparameter tuning in group LASSO}

We construct a synthetic regression task in which the design matrix $\mathbf{X}\in\mathbb{R}^{n_{\mathrm{train}}\times p}$ is partitioned into $K = d$ groups of $g = 5$ contiguous columns each, so $p = K g$. Within each group, columns are noisy copies of a shared latent factor with intra-group correlation $\rho = 0.9$; only the first $K_{\mathrm{inf}} = 1$ group carries signal, with its true coefficients drawn from $\mathcal{N}(0, I_g)$, while the remaining $K-1$ groups are pure noise. Targets are generated as $y = \mathbf{X}\boldsymbol{\beta}^{\star} + \varepsilon$ with $\varepsilon \sim \mathcal{N}(0,1)$, and the data is split into $n_{\mathrm{train}} = 50$ training and $n_{\mathrm{test}} = 150$ test samples, standardized using the training statistics. We vary the hyperparameter dimension $d \in \{2, 10, 20\}$ and run $n_{\mathrm{seeds}} = 5$ independent seeds per setting.

The hyperparameter is the per-group log-regularization vector $\boldsymbol{\theta} = (\theta_1, \ldots, \theta_d)$ with $\theta_k = \log \lambda_k$ and box constraints $\theta_k \in [-6,\, 0]$. For each query, the inner solver runs $200$ iterations of ISTA on the group LASSO objective
\[
\min_{\mathbf{w}\in\mathbb{R}^{p}}\;
\frac{1}{2 n_{\mathrm{train}}} \|\mathbf{X}^{\mathrm{train}} \mathbf{w} - \mathbf{y}^{\mathrm{train}}\|_2^2
\;+\; \sum_{k=1}^{d} e^{\theta_k} \, \|\mathbf{w}_{G_k}\|_2,
\]
and the per-user loss is the half squared residual on each held-out test point. All methods are run for $T = 20$ outer iterations with inner batch size $b = d+1$, gradient clip norm $B = 0.1$, AdaGrad step size $\eta = 1$, and an RBF kernel of lengthscale $\ell = 1$. The DP arm uses $\mu$-GDP with $\mu = 1$; the GIBO baseline runs the same algorithm with $\mu = \infty$ (no privacy noise); random search samples uniformly from the same box $[-6, 0]^d$ for the matched evaluation budget $T(d+1)$. DP-GIBO and GIBO are initialized at $\boldsymbol{\theta}_0$ drawn uniformly from the same box.

\subsection{Hyperparameter tuning in Gaussian Process
Regression}

The results corresponding to the additional experiments can be found in Figure \ref{fig:main-experiments} (a)-(d). All data in this block was generated from a $15$-dimensional GP with RBF kernel; the optimizable hyperparameters consist of $15$ length scales, such that $\Theta \subset \mathbb R^{15}$.

In the first experiment, we varied the bias control $\varepsilon$ from a low to a high value, keeping all other model parameters the same. The results are shown in \Cref{fig:main-experiments} (a) and (b). These plots agree with our theory: higher $\varepsilon$ leads to a lower number of function evaluations required to complete all iterations, but the final attained loss will be worse. As mentioned in the discussion of Theorem \ref{thm:Goodconvergence}, decreasing $\varepsilon$ past the point where the privacy error dominates bias error ceases to improve performance; this can be seen in $\varepsilon = 0.3$ and $\varepsilon = 0.5$ producing roughly the same final loss. 

In our next experiment, displayed in \Cref{fig:main-experiments} (c), we varied the privacy level from very strong privacy, to strong privacy, to no privacy. It is clear that in our moderate $n$ ($700$) example, high levels of privacy come at a significant suboptimality gap cost.

Finally, in \Cref{fig:main-experiments} (d), we plot the final loss achieved after $45$ iterations of DP-GIBO with \textit{fixed} batch size $b = d+1$ for various levels of evaluation noise. The plot indicates that the noiseless and low noise regimes perform similarly, whereas the large noise regime suffers. This confirms our theoretical findings in Theorem \ref{theorem:functionEvalsNoisy}, which suggested that low enough noise produces the same performance as zero noise. This also implies that DP training can come for free as long as the induced noise remains small.

Throughout this example, we used RBF kernel with lengthscale equal to 1. The studied dimension was \( d = 15\), total number of data points was \( n = 2000\) with \( 50-50\) split between training and validation data.

In the example when examining the performance on different privacy levels \(\mu\), we used noise level \( \sigma = 0.05\), bias control level \( \varepsilon = 0.5\), clipping constant \( B = 3\), and we ran the optimization for \( T = 25\) steps.

In the example when examining the performance on different noise levels \(\sigma\), we used the full batch size \(b = d + 1\), clipping constant \( B = 1\), privacy level was set to \( \mu = 0.1\), and we ran the optimization for \( T = 45\) steps.

In the example when examining the performance on different bias control levels \(\varepsilon\), we used clipping constant \( B = 3\), noise level \( \sigma = 0.05 \), privacy level was set to \( \mu = 1\), and we ran the optimization for \( T = 25\) steps. 

For all of the methods, we do 10 independent replications, with random restarts inside the optimization area. 
Finally, instead of doing standard descent update, we have used AdaGrad updates with our privatized estimate of the gradient with \( \eta = 0.3 \).

\begin{figure*}[t]
    \centering
    \includegraphics[width=\textwidth]{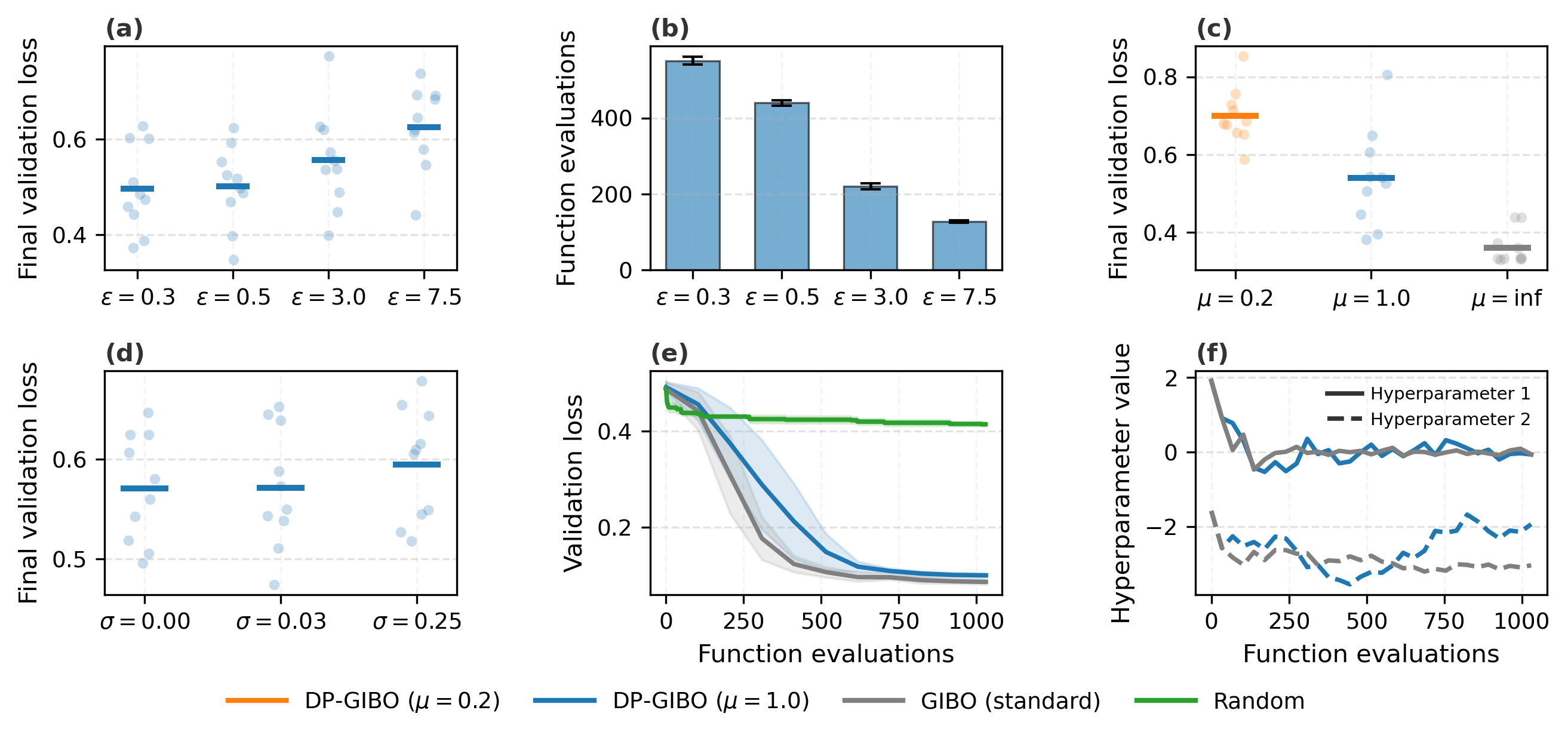}
    \caption{\textit{Results of experiments.}
    \textbf{(a)-(d).} We present the results of Example \ref{example:GP}, where we test the performance of Algorithm \ref{algorithm} on tuning the GP regression lengthscales in $d = 15$ dimensions, and we vary the level of permitted bias \( \varepsilon \), privacy level \( \mu \), and noise level \( \sigma \). 
    \textbf{(e)-(f).} In Example \ref{sec:svmtuning}, we compare our privatized method to non-private GIBO (i.e.\ \( \mu = \infty \)) and to non-private random search.}
    \label{fig:main-experiments}
\end{figure*}

\subsection{Hyperparameter tuning of an SVM}
We tune a support-vector regressor (SVR) on the \emph{Relative Location of CT Slices} dataset~\cite{relative_location_of_ct_slices_on_axial_axis_206}, drawing $n_{\mathrm{train}} = 1000$ training and $n_{\mathrm{val}} = 9000$ validation examples uniformly at random and restricting attention to a random subset of $p = 100$ input features; the full-dimensional problem is impractical given the cost of refitting the SVR at every query. The
hyperparameter vector to be optimised is
\[
    \boldsymbol{\theta} \;=\;
    \big(\log\ell_1,\,\ldots,\,\log\ell_p,\;\log\varepsilon,\;\log C,\;\log\gamma\big)
    \;\in\;\mathbb{R}^{d}, \qquad d \;=\; p + 3 \;=\; 103,
\]
where $\ell_i$ are the per-feature log length-scales of the SVR's RBF kernel and $(\varepsilon, C, \gamma)$ are the $\varepsilon$-tube width, the regularisation strength, and the kernel bandwidth, respectively. Each coordinate is optimised in log-space over a box:  $\log\ell_i \in [-2, 2]$, $\varepsilon \in [0.01,\, 1.0]$, 
$C \in [0.1,\, 3.0]$, and $\gamma \in [0.01,\, 5.0]$. The per-user contribution is the squared error on a single held-out validation example, and the figure of merit minimised by every method is the mean squared error over the validation set.

We perform $n_{\mathrm{seeds}} = 5$ independent replications, each initialised by sampling $\boldsymbol{\theta}_0$ uniformly from the optimisation domain. All methods are run for $T = 10$ outer iterations with inner batch size $b = d + 1 = 104$ inducing points; AdaGrad updates with base step size $\eta = 0.8$ and per-user gradient clip $B = 1$ are applied to the (noised) gradient estimate. The differentially private arm uses $\mu$-Gaussian DP with $\mu = 0.2$, while the ``GIBO'' baseline applies the same algorithm with $\mu = \infty$ (no noise), and random search samples $\boldsymbol{\theta}$ uniformly over the same box for a matched evaluation budget of $T \cdot b$ queries. For the GP surrogate used inside DP-GIBO and GIBO we adopt an RBF kernel with lengthscale equal to $1$.

We also show, in \Cref{fig:main-experiments} (f), the optimization path of two parameters during one run of both the private and non-private algorithms, showing that the estimates converge to an optimum and then oscillate around it, as typically occurs when optimizing with noisy gradient descent.

\subsection{Sensitivity to specification of $\sigma$}
\label{example:noise}
We once again consider the Gaussian Process Regression setup from Example \ref{example:GP}. The only difference here is that we do not observe the actual validation losses \( f \) but instead observe \( f + \lambda Z \), where \( Z \sim N(0, 1) \) and we set \( \lambda = 0.01 \). Thus, the true noise that appears in our function evaluations has standard deviation $0.01$.  We then proceed to privately (\(\mu=1\)) find the best hyperparameter using Algorithm \ref{algorithm}, adapted for the noisy case as described in Section \ref{section:noisy}.  

We run our algorithm three times: first by supplying it with the true value of noise standard deviation (\(\sigma = 0.01\)), then with an overestimate (\(\sigma = 1.0\)), and finally with an underestimate (\(\sigma = 0.0001\)). As shown in Figure \ref{fig:noiseGP}, the correct choice of the noise standard deviation parameter results in the best average performance. The \(\sigma = 1.0\) setting makes our estimate conservative but still viable, whereas underestimating the noise parameter leads to worse performance than a random search.

\begin{figure}[H]
    \centering
    \includegraphics[width=.9\textwidth]{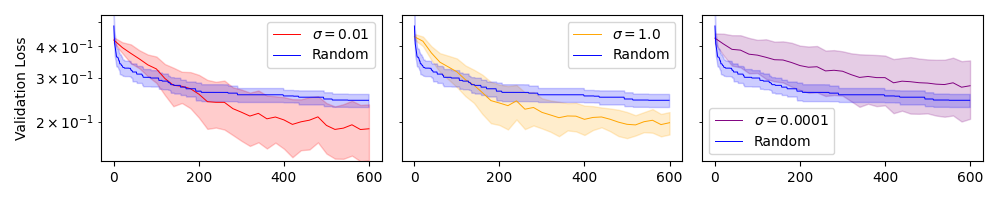}
    \caption{\textbf{Left:} We used the correct variance of the noise; \textbf{Middle:} We overestimated the variance of the noise; \textbf{Right:} We underestimated the variance of the noise}
    \label{fig:noiseGP}
\end{figure}

\subsection{Sensitivity to the Clipping Parameter $B$}\label{sensitivityB}
Finally, we consider the Gaussian Process Regression problem from Example \ref{example:GP}, and set $n = 1000$, $\mu = 1$, $\sigma = 0.05$, and $\varepsilon = 0.5$. We vary the clipping threshold $B$ over four values with a wide spread, in order to understand the sensitivity of our algorithm to the particular clipping threshold chosen, and we again use AdaGrad step-size updates (where we use the privatized gradients so that privacy is maintained by the Post-Processing property of $\mu$-GDP). The results of this sensitivity study can be found in Figure \ref{fig:clippingthreshold}.
\begin{figure}[H]
    \centering
    \includegraphics[width=0.9\linewidth]{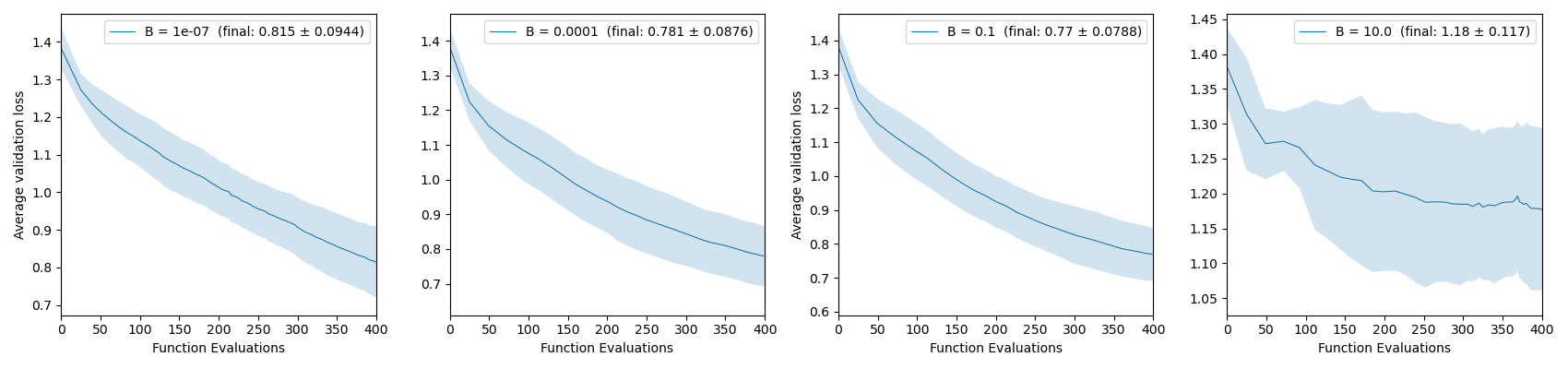}
    \caption{Validation loss of the Gaussian Process Hyperparameter Tuning experiment for various values of $B$}
    \label{fig:clippingthreshold}
\end{figure}

\end{document}